\documentclass[transmag]{IEEEtran}
\usepackage{latexsym}
\usepackage{graphicx}
\usepackage{amsfonts,amssymb,amsmath}
\usepackage{hyperref}

\usepackage{amsthm}
\usepackage{amsmath}
\usepackage{stmaryrd}
\usepackage{algorithm}
\usepackage[noend]{algpseudocode}
\DeclareMathOperator*{\argmax}{arg\,max}
\DeclareMathOperator*{\argmin}{arg\,min}
\usepackage[caption=false,font=footnotesize]{subfig}
\usepackage[noadjust]{cite}
\usepackage{booktabs}
\newtheorem{theorem}{Theorem}

\def\BibTeX{{\rm B\kern-.05em{\sc i\kern-.025em b}\kern-.08em T\kern-.1667em\lower.7ex\hbox{E}\kern-.125emX}}
\begin{document}

\title{Scalable Spectral Clustering with Nystr\"om Approximation: Practical and Theoretical Aspects}

\author{Farhad Pourkamali-Anaraki,
\IEEEmembership{Member, IEEE}
\thanks{F.~Pourkamali-Anaraki is with the Department of Computer Science, University of Massachusetts, Lowell, MA 01854 USA (e-mail: farhad\_pourkamali@uml.edu). IEEE Open Journal of Signal Processing, vol. 1, pp. 242--256, 2020. DOI: \href{http://dx.doi.org/10.1109/OJSP.2020.3039330}{10.1109/OJSP.2020.3039330}}}

\IEEEtitleabstractindextext{\begin{abstract}Spectral clustering techniques are valuable tools in signal processing and machine learning for partitioning complex data sets. The effectiveness of spectral clustering stems from constructing a non-linear embedding based on creating a similarity graph and computing the spectral decomposition of the Laplacian matrix. However, spectral clustering methods fail to scale to large data sets because of high computational cost and memory usage. A popular approach for addressing these problems utilizes the Nystr\"om method, an efficient sampling-based algorithm for computing low-rank approximations to large positive semi-definite matrices. This paper demonstrates how the previously popular approach of Nystr\"om-based spectral clustering has severe limitations. Existing time-efficient methods ignore critical information by prematurely reducing the rank of the similarity matrix associated with sampled points. Also, current understanding is limited regarding how utilizing the Nystr\"om approximation will affect the quality of spectral embedding approximations. To address the limitations, this work presents a principled spectral clustering algorithm that exploits spectral properties of the similarity matrix associated with sampled points to regulate accuracy-efficiency trade-offs. We provide theoretical results to reduce the current gap and present numerical experiments with real and synthetic data. Empirical results demonstrate the efficacy and efficiency of the proposed method compared to existing spectral clustering techniques based on the Nystr\"om method and other efficient methods. The overarching goal of this work is to provide an improved baseline for future research directions to accelerate spectral clustering.\end{abstract}


}

\maketitle

\section{INTRODUCTION}
Cluster analysis is a fundamental problem in signal processing and exploratory data analysis that divides a data set into several groups using the information found only in the data. Among several techniques \cite{saxena2017review,rodriguez2019clustering}, spectral clustering \cite{ng2002spectral,von2007tutorial} is one of the most prominent and successful methods to capture complex structures, such as non-spherical clusters. In these scenarios, spectral clustering  outperforms popular Euclidean clustering techniques, such as K-means clustering \cite{pourkamali2017preconditioned,vijayaraghavan2017clustering}. Hence, spectral clustering has found applications in various domains, including computer vision \cite{shi2000normalized,tacsdemir2015approximate,chen2017linear}, biology \cite{shi2017spectral}, neuroscience \cite{thirion2014fmri}, recommender systems \cite{li2019new}, and blind source separation \cite{bach2006learning}. 

Spectral clustering expresses data clustering as a graph partitioning problem by constructing an undirected similarity graph with each point in the data set being a node. A popular connectivity measure employs the radial basis kernel function of the form $\kappa(\mathbf{x}_i,\mathbf{x}_j)=\exp(-\|\mathbf{x}_i-\mathbf{x}_j\|_2^2/\sigma^2)$,
where $\mathcal{X}=\{\mathbf{x}_1,\ldots,\mathbf{x}_n\}$ is the set of $n$  data points in $\mathbb{R}^d$ to be partitioned into $k$ clusters, and $\sigma>0$ is the bandwidth parameter. Thus, the first step of spectral clustering involves forming a positive semi-definite kernel matrix $\mathbf{K}\in\mathbb{R}^{n\times n}$ with the $(i,j)$-th entry $[\mathbf{K}]_{ij}=\kappa(\mathbf{x}_i,\mathbf{x}_j)$, which describes similarities among $n$ input data points. Therefore, a significant challenge in applying spectral clustering is the computation and storage of the entire kernel matrix, which requires $\mathcal{O}(n^2d)$ time and $\mathcal{O}(n^2)$ space. The quadratic complexity in the number of input data points renders spectral clustering intractable for large data sets.

In this work, we focus on the popular method of  normalized cut \cite{shi2000normalized,he2016iterative} to partition the resulting similarity graph.  This method forms the normalized Laplacian matrix as follows:
\begin{equation}
	\mathbf{L}=\mathbf{D}^{-1/2}(\mathbf{D}-\mathbf{K})\mathbf{D}^{-1/2}=\mathbf{I}_n - \mathbf{D}^{-1/2}\mathbf{K}\mathbf{D}^{-1/2},\label{eq:laplacian}
\end{equation}
where $\mathbf{D}=\text{diag}(\mathbf{K}\mathbf{1}_n)\in\mathbb{R}^{n\times n}$ is the diagonal degree matrix associated with $\mathbf{K}$, $\mathbf{I}_n$ is the identity matrix, and $\mathbf{1}_n$ is the vector of all ones. The spectral embedding of the original data $\mathcal{X}$ is then obtained by solving the following trace minimization problem:
\begin{equation}
	\argmin_{\mathbf{H}\in\mathbb{R}^{n\times k}, \; \mathbf{H}^T\mathbf{H}=\mathbf{I}_k} \text{tr}\big(\mathbf{H}^T\mathbf{L}\mathbf{H}\big),\label{eq:eigVec}
\end{equation}
where $\text{tr}(\cdot)$ represents the matrix trace. The above optimization problem has a closed-form solution; its optimizer is obtained by eigenvectors corresponding to the $k$ smallest eigenvalues of the Laplacian matrix $\mathbf{L}$. Thus, spectral clustering learns a non-linear map that embeds the original data into the eigenspace of $\mathbf{L}$ for uncovering the intrinsic structure of the input data. 

We can also view the spectral embedding process as solving a low-rank approximation problem \cite{zhu2018global}. To this end, we substitute $\mathbf{L}$ from \eqref{eq:laplacian} into the above minimization problem. Using the constraint $\mathbf{H}^T\mathbf{H}=\mathbf{I}_k$ allows rewriting the minimization problem in \eqref{eq:eigVec} as follows:
\begin{equation}
	\argmax_{\mathbf{H}\in\mathbb{R}^{n\times k},\; \mathbf{H}^T\mathbf{H}=\mathbf{I}_k} \text{tr}\big(\mathbf{H}^T\mathbf{M}\mathbf{H}\big),\label{eq:modLap}
\end{equation}
where we used the fact that $\text{tr}(\mathbf{H}^T\mathbf{H})=k$ is a constant, and introduced the modified kernel matrix:
\begin{equation}
\mathbf{M}:=\mathbf{D}^{-1/2}\mathbf{K}\mathbf{D}^{-1/2}\in\mathbb{R}^{n\times n}.\label{eq:M}
\end{equation}
 The solution of this maximization problem is obtained by computing eigenvectors corresponding to the $k$ largest eigenvalues of $\mathbf{M}$, denoted by $\mathbf{U}_{\mathbf{M},k}\in\mathbb{R}^{n\times k}$. That is, we should compute the rank-$k$ approximation of $\mathbf{M}$ to map the original data from $\mathbb{R}^d$ into $\mathbb{R}^k$. The computational complexity associated with finding the  $k$ leading eigenvectors of the modified kernel matrix is $\mathcal{O}(n^2k)$ without making assumptions on the structure of the kernel matrix, such as $\mathbf{K}$ being sparse \cite{halko2011finding}. When the number of data points $n$ is large, the exact solution of this step becomes computationally prohibitive and suffers from large memory overhead. 

The third step of spectral clustering is to partition the $n$ rows of $\mathbf{U}_{\mathbf{M},k}$ using K-means clustering \cite{bachem2018scalable}, where the goal is to find  $k$ centroids and assign each embedded point to the closest centroid for partitioning the original data. Although solving this problem is NP-hard \cite{drineas1999clustering}, it is common to use iterative algorithms that lead to expected approximation guarantees \cite{kmeansplus,ahmadian2019better}. As these algorithms should compute distances between $n$ data points and $k$ centroids in $\mathbb{R}^k$,  the third step requires $\mathcal{O}(n k^2)$ operations per iteration, and a few tens of iterations typically suffice to cluster the data. Since the third step takes linear time in the number of data points, the central challenge that arises in large-scale data settings is computing the leading eigenvectors of the modified kernel matrix $\mathbf{M}$, which is the main focus of this paper.   

While other variants exist in the literature, including robust methods to noisy data \cite{bojchevski2017robust,kim2020outer}, this work focuses on the standard formulation of spectral clustering. Alg.~\ref{alg:SC} summarizes the three main steps of prototypical spectral clustering. Empirical evidence suggests that normalizing the rows of $\mathbf{U}_{\mathbf{M},k}$ improves stability and accuracy \cite{von2007tutorial}. As these row vectors are in $\mathbb{R}^k$, the normalization cost scales linearly in terms of the number of data points. Using cross-validation is a common technique for selecting the kernel parameter $\sigma$. 

\begin{algorithm}
     \textbf{Input:}  data set $\mathcal{X}=\{\mathbf{x}_1,\ldots,\mathbf{x}_n\}$, kernel parameter $\sigma$, number of clusters $k$.
\begin{algorithmic}[1]
	\Function{SC}{$\mathcal{X}, \sigma, k$}
	\State Form the kernel matrix $\mathbf{K}\in\mathbb{R}^{n\times n}$;\label{alg:step1}
	\State Construct $\mathbf{M}=\mathbf{D}^{-1/2}\mathbf{K}\mathbf{D}^{-1/2}\in\mathbb{R}^{n\times n}$, where $\mathbf{D}=\text{diag}(\mathbf{K}\mathbf{1}_n)$;
	\State Compute the $k$  leading eigenvectors of $\mathbf{M}$ to obtain the spectral embedding $\mathbf{U}_{\mathbf{M},k}\in\mathbb{R}^{n\times k}$;
	\State Normalize each row of $\mathbf{U}_{\mathbf{M},k}$ to have unit length;
	\State Perform K-means clustering over the rows of the normalized matrix $\mathbf{U}_{\mathbf{M},k}$;
	
	{
	\Return Clustering results.
}
	\caption{Prototypical Spectral Clustering (SC)}\label{alg:SC}
	\EndFunction
\end{algorithmic}
\end{algorithm}

\subsection{Related Work on Accelerating Spectral Clustering}
Various methods have been proposed to accelerate spectral clustering by computing an approximate spectral embedding of the original data. Recent work \cite{tremblay2020approximating} presented an excellent review of the literature on this topic for interested readers. In this paper, we divide the related work into two main categories: (1) methods that circumvent the computation of the full kernel matrix, and (2) techniques that consider the similarity graph as one of the inputs to spectral clustering and, thus, ignore the cost associated with step \ref{alg:step1} of Alg.~\ref{alg:SC}. The former is more realistic since constructing full kernel matrices is computationally prohibitive, even for medium-sized data. We further divide the first line of work into three sub-categories:
 \begin{enumerate}
 	\item explicit approximation of the kernel matrix;
 	\item random Fourier features to approximate the radial basis kernel function; 
 	\item forming a sparse similarity graph. 
 \end{enumerate}

The Nystr\"om method is one of the most popular techniques for approximating positive semi-definite matrices. In a nutshell, the Nystr\"om method \cite{williams2001using} selects $m<n$ points from the original data set $\mathcal{X}$ using a sampling strategy, such as uniform sampling or a more complicated non-uniform selection technique \cite{kumar2012sampling,sun2015review,pourkamali2018randomized}. After choosing a subset of the data, the so-called \textit{landmarks} which we denote them by $\mathcal{Z}=\{\mathbf{z}_1,\ldots,\mathbf{z}_m\}$, one should compute similarities between the original data $\mathcal{X}$ and $\mathcal{Z}$, as well as pairwise similarities among the elements of  $\mathcal{Z}$. Hence, the Nystr\"om method constructs two matrices $\mathbf{C}\in\mathbb{R}^{n\times m}$ and $\mathbf{W}\in\mathbb{R}^{m\times m}$ such that $[\mathbf{C}]_{ij}=\kappa(\mathbf{x}_i,\mathbf{z}_j)$ and $[\mathbf{W}]_{ij}=\kappa(\mathbf{z}_i,\mathbf{z}_j)$, which takes $\mathcal{O}(nmd)$ time. We then obtain an approximation of the kernel matrix in the form of:
\begin{equation}
\mathbf{K}\approx\widehat{\mathbf{K}}=\mathbf{C}\mathbf{W}^\dagger\mathbf{C}^T,
\end{equation}
where $\mathbf{W}^\dagger$ is the pseudo-inverse of $\mathbf{W}$. 

When employing the Nystr\"om method, the number of landmarks should exceed $k$, i.e., the desired number of clusters, and increasing $m$ is a common practice to improve accuracy \cite{wang2019scalable,pourkamali2019improved}. The Nystr\"om approximation results in linear time complexity in the size of the original data for a fixed landmark set \cite{bach2013sharp}. However, a critical task is to efficiently integrate the Nystr\"om method with spectral clustering because the ultimate goal is to estimate the  $k$ leading eigenvectors of the modified kernel matrix $\mathbf{M}$  in linear time (instead of $\mathbf{K}$).  Therefore, several variants of Nystr\"om-based spectral clustering have been proposed \cite{fowlkes2004spectral,li2011time,mohan2017exploiting}, where the underlying theme is to compute the spectral decomposition of the inner matrix $\mathbf{W}$ to save on computational resources and lift the solution from $\mathbb{R}^m$ back to $\mathbb{R}^n$. Although these methods reduce time complexity, a downside is the lack of a theoretical framework and understanding concerning how the Nystr\"om method affects the quality of resulting spectral embedding. 

The second sub-category seeks to directly approximate the radial basis kernel function in the form of  $\kappa(\mathbf{x}_i,\mathbf{x}_j)\approx \phi(\mathbf{x}_i)^T\phi(\mathbf{x}_j)$, where $\phi(\mathbf{x})\in\mathbb{R}^D$ is known as the random Fourier feature vector \cite{rahimi2008random}. The main idea behind this approach is to use the Fourier transform of shift-invariant kernel functions, including Gaussian kernels, for the efficient computation of feature vectors \cite{sutherland2015error,wu2016revisiting,he2018fast}. Recent work \cite{wu2018scalable} utilized this strategy to implicitly approximate the kernel matrix $\mathbf{K}$ for performing spectral clustering in linear time. The introduced method requires the dimension of feature vectors $D$ to be significantly greater than the ambient dimension $d$. However, each feature vector should contain only a few non-zero entries to reduce the subsequent eigenvalue decomposition cost. 

The third sub-category utilizes fast nearest neighbor search algorithms \cite{muja2014scalable} to form similarity graphs with sparse kernel matrices, which may substantially reduce the computational and memory complexities associated with the spectral decomposition step \cite{dong2011efficient,lucinska2012spectral}. In this case, the number of nearest neighbors is an additional tuning parameter that remarkably impacts the connectivity of the similarity graph and the following spectral embedding. 

The second class of accelerated spectral clustering techniques assumes that the similarity graph is one of the inputs. Therefore, the main task is to compute the eigenvalue decomposition of the modified kernel matrix $\mathbf{M}$, which is assumed to be available at no cost. A possible solution focuses on utilizing tools from the randomized numerical linear algebra literature, such as randomized subspace iteration \cite{saibaba2019randomized}. These methods typically employ random projections to identify a subspace that approximately captures the range of $\mathbf{M}$ \cite{boutsidis2015spectral}. Another approach seeks to form a sparse similarity matrix according to the effective resistances of all nodes \cite{spielman2011graph}, which can be approximated in nearly linear time. However, these techniques are practical only when kernel matrices are accessible.

\subsection{Main Contributions}
In this paper, we design and study an efficient method for incorporating the Nystr\"om approximation into prototypical spectral clustering. The main feature of the proposed approach is to exploit spectral properties of the inner matrix $\mathbf{W}$, allowing us to regulate accuracy-efficiency trade-offs. Hence, our approach is suitable for clustering complex data sets containing tens or hundreds of thousands of samples. 

As we will discuss in detail, efficient Nystr\"om-based spectral clustering methods take a two-step approach, which entails restricting the rank of the kernel matrix associated with the landmark set, followed by lifting the solution back to the original space. The disentanglement of the two similarity matrices $\mathbf{C}$ and $\mathbf{W}$ for computing the spectral decomposition gives rise to several issues. First, performing the rank reduction step too early adversely affects the spectral embedding process. Second, the produced eigenvectors are not necessarily orthogonal, which will require additional orthogonalization steps. Third, providing theoretical guarantees to understand the relationship between the Nystr\"om approximation error and the quality of resulting spectral embedding becomes complicated. A serious concern is that a small perturbation of the kernel matrix $\mathbf{K}$ may have an out-sized influence on the modified kernel matrix $\mathbf{M}=\mathbf{D}^{-1/2}\mathbf{K}\mathbf{D}^{-1/2}$ due to the perturbation of the degree matrix $\mathbf{D}$, which is used for normalizing the row and columns of the kernel matrix. 

This work improves  Nystr\"om-based spectral clustering by utilizing both matrices $\mathbf{C}$ and $\mathbf{W}$ at the same time. Our proposed approach automatically exploits decay in the spectrum of the inner matrix $\mathbf{W}$ to regulate accuracy-efficiency trade-offs, instead of enforcing its rank to be $k$ as prescribed by the prior work. We then implicitly form the modified kernel matrix $\mathbf{M}$ and compute its leading eigenvectors.

A further advantage of the proposed approach is reducing the current gap in the literature between a provably good low-rank approximation of the kernel matrix $\mathbf{K}$ to a provably accurate estimation of $\mathbf{M}$.  We derive an upper bound for the perturbation of the modified kernel matrix due to the Nystr\"om method by making use of the Taylor series expansion for matrix functions \cite{deadman2016taylor}. Our analysis shows that a relatively small perturbation of the kernel matrix results in a practical upper bound for approximating the modified kernel matrix, or equivalently the normalized Laplacian matrix. We present numerical experiments to understand the main assumptions and bounds involved in our theoretical results. 

Finally, we present an extensive empirical evaluation of the proposed approach, using both synthetic and real data. We compare our introduced method with other state-of-the-art approximate spectral clustering methods that circumvent the formation of the entire similarity graph, including techniques based on random Fourier features and sparse similarity graphs. We also corroborate the scalability of our proposed spectral clustering method concerning the size of input data and the number of landmarks. 

\subsection{Paper Organization }
The rest of the paper is outlined as follows. We first thoroughly review the related work on Nystr\"om-based spectral clustering and exemplify several drawbacks in Section \ref{sec:rel}. Then, Section \ref{sec:main} explains our proposed scalable approach for improving Nystr\"om-based spectral clustering while providing new theoretical results to reduce the current gap in the literature. We also present numerical experiments to verify the assumptions made in our analysis. Section \ref{sec:exper} empirically demonstrates trade-offs between accuracy and efficiency of the proposed Nystr\"om-based spectral clustering method on benchmark data sets. We present concluding remarks in Section \ref{sec:conc}. 

\subsection{Notation and Preliminaries}
We denote column vectors with lower-case bold letters and matrices with upper-case bold letters. We take $\mathbf{1}_n$ to be the $n$-dimensional vector of all ones and $\mathbf{I}_n$ represents the $n\times n$ identity matrix. For a vector $\mathbf{x}\in\mathbb{R}^n$, let $\|\mathbf{x}\|_2$ be the Euclidean norm and  $\text{diag}(\mathbf{x})$ returns a diagonal matrix with the elements of $\mathbf{x}$ on the main diagonal. Given a matrix $\mathbf{A}\in\mathbb{R}^{n\times m}$, the $(i,j)$-th element is denoted by $[\mathbf{A}]_{ij}$ and $\mathbf{A}^T$ is the transpose of $\mathbf{A}$. The matrix $\mathbf{A}$ admits a factorization, known as the truncated singular value decomposition (SVD), in the form of $\mathbf{A}=\mathbf{U}_{\mathbf{A}}\boldsymbol{\Sigma}_{\mathbf{A}}\mathbf{V}_{\mathbf{A}}^T$, where $\mathbf{U}_{\mathbf{A}}\in\mathbb{R}^{n\times r}$ and $\mathbf{V}_{\mathbf{A}}\in\mathbb{R}^{m\times r}$ are matrices with orthonormal columns referred to as the left singular vectors and right singular vectors, respectively. The parameter $r<\min\{n,m\}$ represents the rank of $\mathbf{A}$ and the diagonal matrix $\boldsymbol{\Sigma}_{\mathbf{A}}=\text{diag}([\sigma_1(\mathbf{A}),\ldots,\sigma_r(\mathbf{A})])$ contains the singular values of $\mathbf{A}$ in descending order, i.e., $\sigma_1(\mathbf{A})\geq \ldots\geq \sigma_r(\mathbf{A})>0$. In this paper, the expression ``$\mathbf{A}$ has rank $r$'' means that the rank of $\mathbf{A}$ does not exceed $r$. Using this factorization, we can define several standard matrix norms, including the Frobenius norm $\|\mathbf{A}\|_F^2:=\sum_{i=1}^{r}\sigma_i(\mathbf{A})^2$ and the spectral norm $\|\mathbf{A}\|_2:=\sigma_1(\mathbf{A})$. 

If $\mathbf{A}\in\mathbb{R}^{n\times n}$ is symmetric positive semi-definite, we have $\mathbf{U}_{\mathbf{A}}=\mathbf{V}_{\mathbf{A}}$ in the previous factorization, which is called the reduced eigenvalue decomposition (EVD) or spectral decomposition. The columns of $\mathbf{U}_{\mathbf{A}}\in\mathbb{R}^{n\times r}$ are the eigenvectors of $\mathbf{A}$ and $\boldsymbol{\Sigma}_{\mathbf{A}}$ contains the corresponding eigenvalues in descending order. Thus, we get $\mathbf{A}=\mathbf{U}_{\mathbf{A}}\boldsymbol{\Sigma}_{\mathbf{A}}\mathbf{U}_{\mathbf{A}}^T$, where $\mathbf{U}_{\mathbf{A}}^T\mathbf{U}_{\mathbf{A}}=\mathbf{I}_r$. The Moore-Penrose pseudo-inverse of $\mathbf{A}$ can be obtained from the EVD as $\mathbf{A}^\dagger=\mathbf{U}_{\mathbf{A}}\boldsymbol{\Sigma}_\mathbf{A}^{-1}\mathbf{U}_{\mathbf{A}}^T$, where $\boldsymbol{\Sigma}_\mathbf{A}^{-1}=\text{diag}([\sigma_1(\mathbf{A})^{-1},\ldots,\sigma_r(\mathbf{A})^{-1}])$. When $\mathbf{A}$ is full rank, i.e., $r=n$, we have $\mathbf{A}^\dagger=\mathbf{A}^{-1}$. The trace of $\mathbf{A}$ is equal to the sum of its eigenvalues, i.e., $\text{tr}(\mathbf{A})=\sum_{i=1}^{r}\sigma_i(\mathbf{A})$. The matrix $\mathbf{A}$ is positive semi-definite of rank $r$ if and only if there exists a matrix $\mathbf{B}$ of rank $r$ such that $\mathbf{A}=\mathbf{B}\mathbf{B}^T$. 

In this paper, given an integer $k$ that does not exceed the rank parameter $r$, we denote the first $k$ columns of $\mathbf{U}_{\mathbf{A}}$, i.e., the $k$ leading eigenvectors of $\mathbf{A}$, by $\mathbf{U}_{\mathbf{A},k}\in\mathbb{R}^{n\times k}$. Similarly, $\boldsymbol{\Sigma}_{\mathbf{A},k}\in\mathbb{R}^{k\times k}$ represents a diagonal sub-matrix that contains the $k$ largest eigenvalues of $\mathbf{A}$. Based on the Eckart-Young-Mirsky theorem, $\llbracket\mathbf{A}\rrbracket_k:=\mathbf{U}_{\mathbf{A},k}\boldsymbol{\Sigma}_{\mathbf{A},k}\mathbf{U}_{\mathbf{A},k}^T$ is the best rank-$k$ approximation to $\mathbf{A}$ because it minimizes the approximation error $\|\mathbf{A}-\widehat{\mathbf{A}}\|$ for any unitarily invariant norm over all matrices $\widehat{\mathbf{A}}$ of rank  $k$. The error incurred in the spectral norm by the best rank-$k$ approximation  can be identified as:
\begin{equation}
\|\mathbf{A}-\llbracket\mathbf{A}\rrbracket_k\|_2=\sigma_{k+1}(\mathbf{A}).
\end{equation}

\section{Review of Spectral Clustering Using Nystr\"om Approximation}\label{sec:rel}
Previous research on integrating the Nystr\"om method with spectral clustering has proposed several techniques to compute an approximate spectral embedding in linear time concerning the data size $n$. The key challenge is how to utilize the Nystr\"om approximation of the kernel matrix, i.e., $\widehat{\mathbf{K}}=\mathbf{C}\mathbf{W}^\dagger\mathbf{C}^T$, to compute the top eigenvectors of the approximate modified kernel matrix $\widehat{\mathbf{M}}=\widehat{\mathbf{D}}^{-1/2}\widehat{\mathbf{K}}\widehat{\mathbf{D}}^{-1/2}$, where $\widehat{\mathbf{D}}=\text{diag}(\widehat{\mathbf{K}}\mathbf{1}_n)$. Obviously, the desired linear complexity in terms of time and space does not permit forming square matrices of size $n\times n$. Thus, 
the previous research focused on constructing a small matrix of size $m\times m$ as a proxy for the kernel matrix $\mathbf{K}$. Although this strategy reduces the cost, a major concern is that critical information regarding the structure of the matrix $\mathbf{C}\in\mathbb{R}^{n\times m}$ may be ignored during this process, adversely impacting the accuracy of spectral clustering. In this section, we review two relevant prior techniques and highlight their limitations to motivate our proposed approach. 

Exploiting the Nystr\"om approximation for estimating the leading eigenvectors of the modified kernel matrix was first introduced in \cite{fowlkes2004spectral}. The main idea behind this approach is to find the exact EVD of an $m\times m$ matrix, where $m$ refers to the number of selected landmarks. Then, a linear transformation from $\mathbb{R}^m$ to $\mathbb{R}^n$ generates an approximate spectral embedding of the original data. Let us write the Nystr\"om approximation $\widehat{\mathbf{K}}=\mathbf{C}\mathbf{W}^{-1}\mathbf{C}^T$ using two sub-matrices $\mathbf{W}$ and $\mathbf{B}$:
\begin{equation}
\widehat{\mathbf{K}}= \begin{bmatrix} \mathbf{W} \\ \mathbf{B}^T\end{bmatrix}\mathbf{W}^{-1}\begin{bmatrix}\mathbf{W}&\mathbf{B} \end{bmatrix}=\begin{bmatrix} \mathbf{W} & \mathbf{B}\\\mathbf{B}^T& \mathbf{B}^T\mathbf{W}^{-1}\mathbf{B}\end{bmatrix},
\end{equation}
where  $\mathbf{W}\in\mathbb{R}^{m\times m}$ representing similarities between $m$ distinct landmark points using the Gaussian kernel function has full rank \cite{scholkopf2001learning}, which allows calculating the inverse of $\mathbf{W}$ (this discussion also holds when using pseudo-inverse). Similarly,  the information regarding connectivity measures between $\mathcal{Z}$ and the remaining data points $\mathcal{X}\setminus\mathcal{Z}$  is encoded in $\mathbf{B}\in\mathbb{R}^{m\times (n-m)}$. One can compute and decompose the approximate degree matrix $\widehat{\mathbf{D}}=\text{diag}(\widehat{\mathbf{K}}\mathbf{1}_n)$ into two diagonal sub-matrices as follows: 
\begin{align}
&\widehat{\mathbf{D}}_1 =  \text{diag}(\mathbf{W}\mathbf{1}_m+\mathbf{B}\mathbf{1}_{n-m})\in\mathbb{R}^{m\times m}, \nonumber\\
&\widehat{\mathbf{D}}_2 =    \text{diag}(\mathbf{B}^T\mathbf{1}_m+\mathbf{B}^T\mathbf{W}^{-1}\mathbf{B}\mathbf{1}_{n-m})\in\mathbb{R}^{(n-m)\times (n-m)}.
\end{align}
Thus, we can compute two normalized matrices $\widetilde{\mathbf{W}}:=\widehat{\mathbf{D}}_1^{-1/2}\mathbf{W}\widehat{\mathbf{D}}_1^{-1/2}$ and $\widetilde{\mathbf{B}}:=\widehat{\mathbf{D}}_1^{-1/2}\mathbf{B}\widehat{\mathbf{D}}_2^{-1/2}$.
The next step is to define $\mathbf{R}\in\mathbb{R}^{m\times m}$ and calculate its EVD: 
\begin{equation}
\mathbf{R}:=\widetilde{\mathbf{W}}+\widetilde{\mathbf{W}}^{-1/2}\widetilde{\mathbf{B}}\widetilde{\mathbf{B}}^T\widetilde{\mathbf{W}}^{-1/2}=\mathbf{U}_{\mathbf{R}}\boldsymbol{\Sigma}_{\mathbf{R}}\mathbf{U}_{\mathbf{R}}^T.\label{eq:R}
\end{equation}
The last step is to find the relationship between the spectral decomposition of $\mathbf{R}$ and  the modified kernel matrix $\widehat{\mathbf{M}}$:
\begin{equation}
\widehat{\mathbf{U}}_{\mathbf{M}}^{nys1}:=\begin{bmatrix}\widetilde{\mathbf{W}}\\\widetilde{\mathbf{B}}^T\end{bmatrix}\widetilde{\mathbf{W}}^{-1/2}\mathbf{U}_{\mathbf{R}}\boldsymbol{\Sigma}_{\mathbf{R}}^{-1/2}\in\mathbb{R}^{n\times m},\; \widehat{\boldsymbol{\Sigma}}_{\mathbf{M}}^{nys1}:=\boldsymbol{\Sigma}_{\mathbf{R}}.\label{eq:lift1}
\end{equation}
Thus, the main contribution of \cite{fowlkes2004spectral} was to generate a factorization of $\widehat{\mathbf{M}}$ without explicitly
computing all its entries:
\begin{equation}
\widehat{\mathbf{M}}=\widehat{\mathbf{U}}_{\mathbf{M}}^{nys1} \widehat{\boldsymbol{\Sigma}}_{\mathbf{M}}^{nys1} \big(\widehat{\mathbf{U}}_{\mathbf{M}}^{nys1}\big)^T.
\end{equation}
Although this technique estimates the leading eigenvectors of the modified kernel matrix, a significant downside is the need to use  $\widetilde{\mathbf{B}}$, or equivalently $\mathbf{C}$, more than once to form the matrix $\mathbf{R}$ and map its eigenspace from $\mathbb{R}^m$ to $\mathbb{R}^n$ via \eqref{eq:lift1}. Hence, this method incurs high computational complexity and memory overhead for clustering huge data sets.

\begin{figure*}[t]
	\centering
	\subfloat[][\label{fig:prior_acc} estimation accuracy of leading eigenvectors]{
		\includegraphics[width=0.5\linewidth]{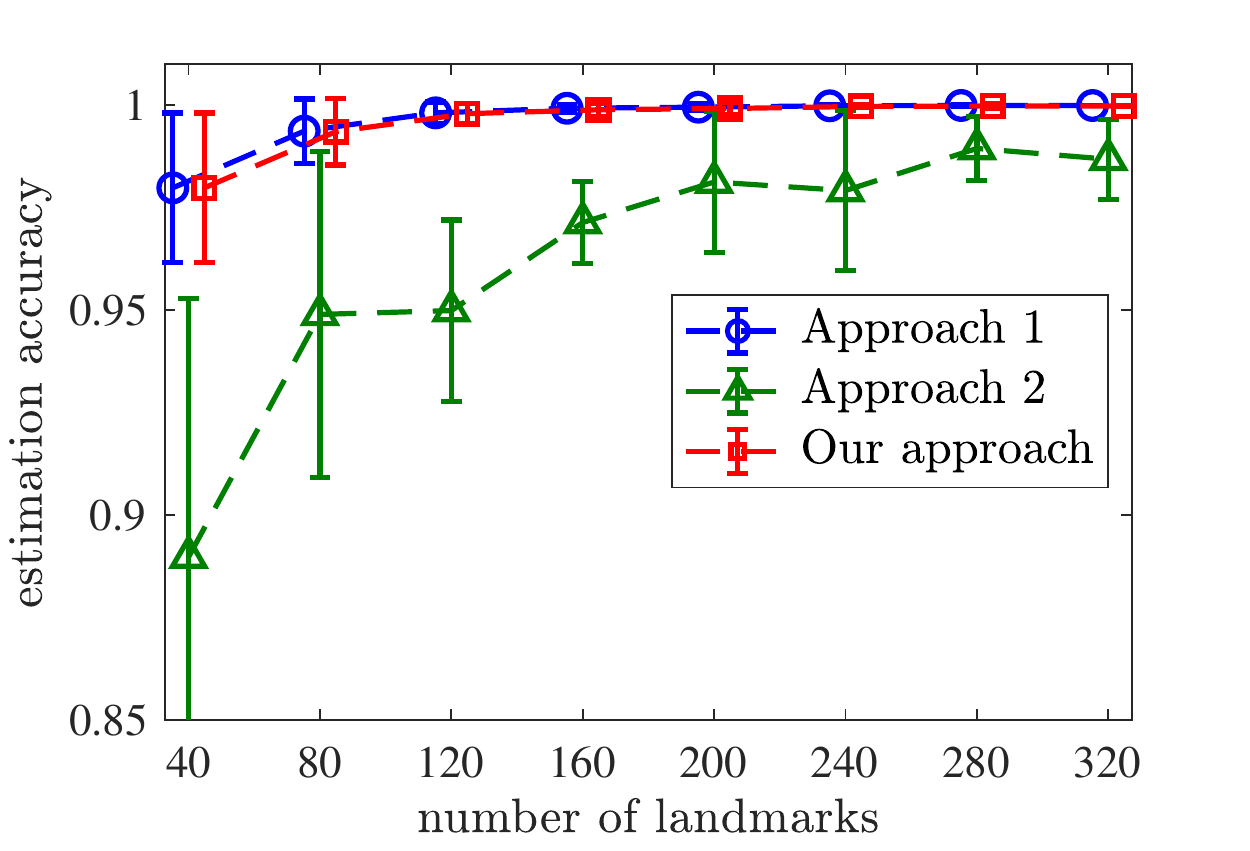}
	}
	\subfloat[][\label{fig:prio_time} running time]{
		\includegraphics[width=0.5\linewidth]{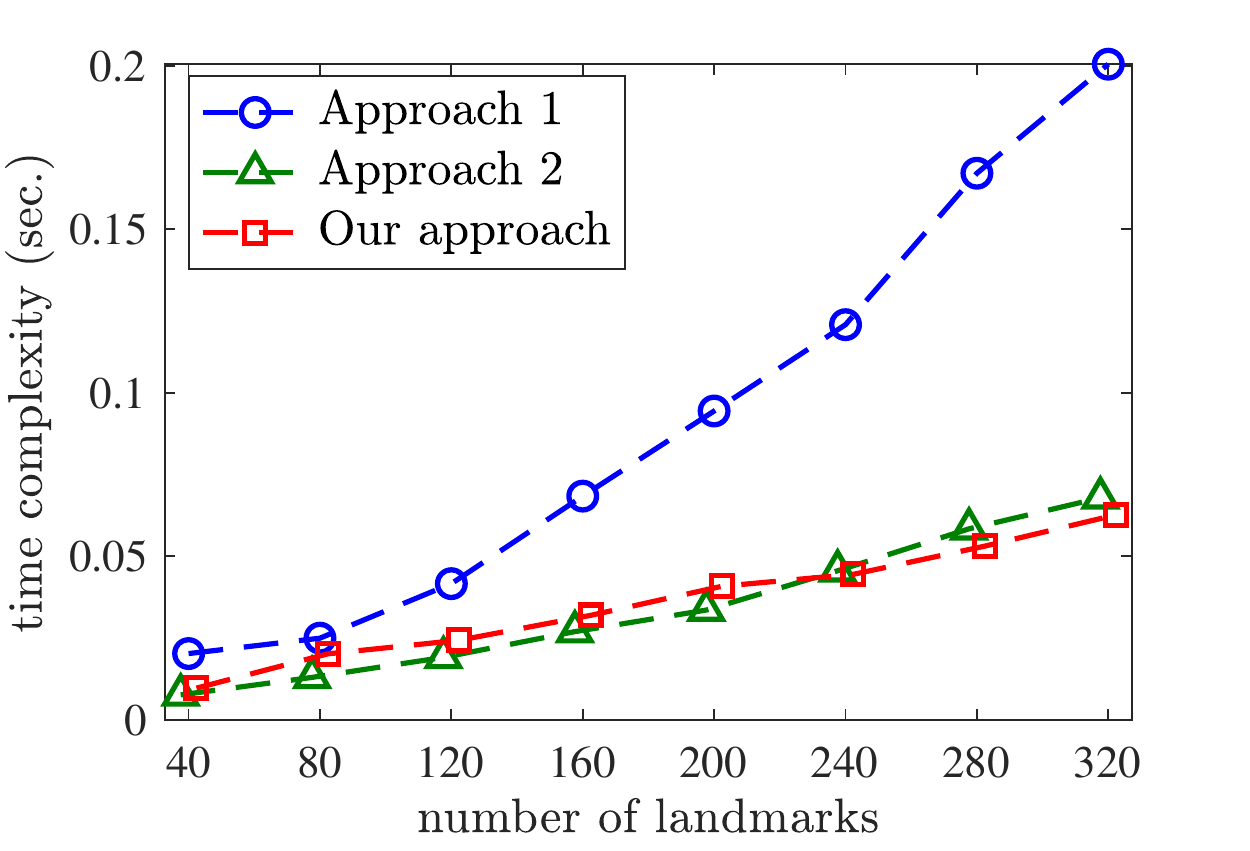}
	} 
	\caption{\label{fig:prior}
		Comparing accuracy and time complexity of our approach with the  previous research (Approach 1 and 2 refer to \cite{fowlkes2004spectral} and \cite{li2011time}, respectively). 
	}
\end{figure*}

In order to improve the scalability of integrating the Nystr\"om method with spectral clustering, another approach was proposed in \cite{li2011time}. The main idea is to treat the inner matrix $\mathbf{W}$ as a proxy for the full kernel matrix. That is, one computes the best rank-$k$ approximation of the normalized matrix $\mathbf{W}$ and then the solution will be lifted from $\mathbb{R}^m$ back to $\mathbb{R}^n$ while using the matrix $\mathbf{C}\in\mathbb{R}^{n\times m}$ only once. To be formal, the first step is to compute the following EVD: 
\begin{equation}
\overline{\mathbf{W}}:=\mathbf{D}_{m}^{-1/2}\mathbf{W}\mathbf{D}_{m}^{-1/2}=\mathbf{U}_{\overline{\mathbf{W}}}\boldsymbol{\Sigma}_{\overline{\mathbf{W}}}\mathbf{U}_{\overline{\mathbf{W}}}^T,
\end{equation}
where $\mathbf{D}_m=\text{diag}(\mathbf{W}\mathbf{1}_m)\in\mathbb{R}^{m\times m}$ is the degree matrix associated with $\mathbf{W}$. This approach then immediately utilizes the best rank-$k$ approximation of $\overline{\mathbf{W}}$ to generate a rank-$k$ approximation of the modified kernel matrix $\widehat{\mathbf{M}}$ as follows: 
\begin{equation}
\widehat{\mathbf{U}}_{\mathbf{M}}^{nys2}:=\mathbf{D}_n^{-1/2}\mathbf{Q}\in\mathbb{R}^{n\times k},\;\widehat{\boldsymbol{\Sigma}}_{\mathbf{M}}^{nys2}:=\boldsymbol{\Sigma}_{\overline{\mathbf{W}},k}\in\mathbb{R}^{k\times k}.\label{eq:eig2}
\end{equation}
In this equation, we defined the following matrix, which consists of multiplication of the matrix $\mathbf{C}$ by the leading eigenvalues and eigenvectors of the normalized matrix $\overline{\mathbf{W}}$:
\begin{equation}
\mathbf{Q}:=\mathbf{C}\mathbf{D}_m^{-1/2}\mathbf{U}_{\overline{\mathbf{W}},k}\boldsymbol{\Sigma}_{\overline{\mathbf{W}},k}^{-1}\in\mathbb{R}^{n\times k},
\end{equation}
and $\mathbf{D}_n=\text{diag}\big(\mathbf{Q}\boldsymbol{\Sigma}_{\overline{\mathbf{W}},k}\mathbf{Q}^T\mathbf{1}_n\big)\in\mathbb{R}^{n\times n}$ is a  diagonal degree matrix, which takes linear time concerning the data size $n$. 

Therefore, unlike the approach in \cite{fowlkes2004spectral}, this method requires a single pass over the matrix $\mathbf{C}$, which is a notable gain when the number of samples $n$ is large. However, this reduction of time complexity leads to two drawbacks. Information loss may occur because of performing the strict rank reduction step on the inner matrix $\mathbf{W}$ without taking into consideration the structure of $\mathbf{C}$. Also, the produced eigenvectors $\widehat{\mathbf{U}}_{\mathbf{M}}^{nys2}$ in \eqref{eq:eig2} are not guaranteed to be orthogonal. The authors in \cite{li2011time} proposed an orthogonalization step to tackle this problem.

In a nutshell, we discussed two related Nystr\"om-based techniques that seek to accelerate the prototypical spectral clustering algorithm. The first approach builds on the exact spectral decomposition of the modified kernel matrix. In contrast, the second approach reduces the computational complexity while compromising accuracy by prematurely restricting the rank of the kernel matrix associated with the landmark set. To further illustrate the merits and limitations of these two approaches, we present a numerical experiment to report both accuracy and time complexity as the number of landmarks increases. In this experiment, we use a benchmark data set from LIBSVM \cite{CC01a}, named mushrooms, which consists of $n=8,\!124$ samples with $d=112$ attributes. In Fig.~\ref{fig:prior}, we report the results of our accuracy and timing comparison as a function of the number of landmarks $m$. Since the landmark selection process involves uniform sampling, we use $50$ independent trials for each value of $m$, and we fix the kernel parameter $\sigma=3.5$. Moreover, the mushrooms data set contains two classes, i.e., $k=2$,  and we thus investigate the accuracy of estimating the two leading eigenvectors of the modified kernel matrix $\mathbf{M}$ as follows \cite{fowlkes2004spectral}:
\begin{equation}
\frac{1}{2}\|\big(\widehat{\mathbf{U}}_{\mathbf{M},2}^{nys}\big)^T\mathbf{U}_{\mathbf{M},2}\|_F^2.
\end{equation}
When the two matrices $\widehat{\mathbf{U}}_{\mathbf{M},2}^{nys}, \mathbf{U}_{\mathbf{M},2}\in\mathbb{R}^{n\times 2}$ comprising the leading eigenvectors are identical, this metric reaches its maximum and is equal to $1$. Also, higher values indicate more accurate estimates of the leading eigenvectors for partitioning the embedded data in the last step of spectral clustering. 

Fig.~\ref{fig:prior_acc} shows that the first approach achieves higher accuracy levels than the second method introduced in \cite{li2011time}. This observation is consistent with our previous discussion because the method proposed in \cite{fowlkes2004spectral} does not enforce a strict rank reduction step. However, according to Fig.~\ref{fig:prio_time}, the first approach suffers from high computational complexity, which is a scalability barrier for clustering massive data sets. 

To be formal, let us compare the time complexity of these two approaches. We ignore the shared cost of forming $\mathbf{C}$ and $\mathbf{W}$ in the Nystr\"om method, which is $\mathcal{O}(nmd)$  arithmetic operations. Moreover, we only report dominant costs involving the number of data points $n$ (e.g., we remove $\mathcal{O}(m^3)$ for the EVD of $m\times m$ matrices). The time complexity of the first approach is $\mathcal{O}(nm^2)$ to form $\mathbf{R}$ and the estimated eigenvectors based on \eqref{eq:R} and \eqref{eq:lift1}. However, the second approach takes $\mathcal{O}(nmk)$ operations  based on \eqref{eq:eig2}. Although the cost of both approaches scales linearly as a function of $n$, the time complexity of the first approach scales quadratically with the number of landmarks $m$. To address this problem, we propose a new approach that provides tunable trade-offs between accuracy and efficiency. As we see in Fig.~\ref{fig:prior}, our approach generates accurate estimates of the leading eigenvectors, and its time complexity is similar to that of the efficient approach in \cite{li2011time}. Another shortcoming of the previous research is the lack of theoretical guarantees for the perturbation analysis of the modified kernel matrix $\mathbf{M}$ and its eigenvectors because of leveraging the Nystr\"om approximation. 

\section{The Proposed  Method}\label{sec:main}
This paper presents a systematic treatment of utilizing the Nystr\"om method for improving the accuracy and scalability of approximate spectral clustering. The proposed approach replaces the kernel matrix $\mathbf{K}$ in Alg.~\ref{alg:SC} with its Nystr\"om approximation and  automatically exploits decay in the spectrum of the inner matrix $\mathbf{W}$. We will  take advantage of the extracted information to construct an approximation of the modified kernel matrix in the form of $\mathbf{M}\approx\widetilde{\mathbf{G}}\widetilde{\mathbf{G}}^T$, where $\widetilde{\mathbf{G}}\in\mathbb{R}^{n\times l}$ and $l$ represents the underlying low-rank structure of $\mathbf{W}$. This strategy allows us to use standard SVD solvers for computing the $k$ leading left singular vectors of $\widetilde{\mathbf{G}}$ to obtain the spectral embedding of the original data in linear time concerning both the number of data points and the size of the landmark set.  We also offer new theoretical results for the perturbation analysis of the (modified) kernel matrix due to exploiting the low-rank structure of $\mathbf{W}$ and utilizing the Nystr\"om method $\mathbf{C}\mathbf{W}^\dagger\mathbf{C}^T$. Furthermore, we present two numerical experiments to verify the assumptions made in the analysis and investigate the obtained upper bounds. While the presented results are based on the spectral norm, our analysis can be easily generalized to other unitarily invariant matrix norms, including the Frobenius norm.
\subsection{The Proposed Scalable Spectral Clustering Algorithm}
The first step of our proposed approach exploits decay in the spectrum of the inner matrix $\mathbf{W}\in\mathbb{R}^{m\times m}$, which is obtained from the Nystr\"om approximation $\widehat{\mathbf{K}}=\mathbf{C}\mathbf{W}^\dagger\mathbf{C}^T$, by computing its spectral decomposition as follows:
\begin{equation}
\mathbf{W}=\mathbf{U}_{\mathbf{W}}\boldsymbol{\Sigma}_{\mathbf{W}}\mathbf{U}_{\mathbf{W}}^T,
\end{equation}
where the matrix $\boldsymbol{\Sigma}_{\mathbf{W}}=\text{diag}([\sigma_1(\mathbf{W}),\ldots,\sigma_m(\mathbf{W})])$ contains the decaying spectrum of $\mathbf{W}$. A key aspect of our approach is that eigenvectors corresponding to small eigenvalues generate very little error for a low-rank approximation of $\mathbf{W}$. Therefore, we propose to retain eigenvectors whose corresponding eigenvalues are above a threshold $0<\gamma<1$:
\begin{equation}
l=\max\Big\{i\in\mathbb{N}:\frac{\sigma_i(\mathbf{W})}{\sigma_1(\mathbf{W})}\geq \gamma,\;i\leq m\Big\}.\label{eq:determinel}
\end{equation} 
The parameter $\gamma$ should be chosen so that the obtained $l$ exceeds the number of clusters $k$. Hence, the first step of our approach utilizes the decaying spectrum of the inner matrix for replacing $\mathbf{W}$ with its best rank-$l$ approximation $\llbracket \mathbf{W}\rrbracket_l=\mathbf{U}_{\mathbf{W},l}\boldsymbol{\Sigma}_{\mathbf{W},l}\mathbf{U}_{\mathbf{W},l}^T$, where $\mathbf{U}_{\mathbf{W},l}\in\mathbb{R}^{m\times l}$ and $\boldsymbol{\Sigma}_{\mathbf{W},l}\in\mathbb{R}^{l\times l}$ are the $l$ leading eigenvectors and eigenvalues, respectively. Therefore, the parameter $\gamma$ allows controlling the amount of spectral energy captured by the low-rank approximation (note that $\|\mathbf{W}\|_2=\sigma_1(\mathbf{W})$):
\begin{equation}
\|\mathbf{W}-\llbracket \mathbf{W}\rrbracket_l\|_2=\sigma_{l+1}(\mathbf{W})< \gamma \|\mathbf{W}\|_2.
\end{equation}

Next, we generate an approximation of $\widehat{\mathbf{K}}$ as follows:
\begin{equation}
\widehat{\mathbf{K}}\approx \mathbf{C}\llbracket\mathbf{W}\rrbracket_l^\dagger\mathbf{C}^T=\mathbf{G}\mathbf{G}^T,\;\mathbf{G}:=\mathbf{C}\mathbf{U}_{\mathbf{W},l}\boldsymbol{\Sigma}_{\mathbf{W},l}^{-1/2}\in\mathbb{R}^{n\times l}. \label{eq:G}
\end{equation}
The next step is to find the diagonal degree matrix $\widehat{\mathbf{D}}$  without explicitly computing $\mathbf{G}\mathbf{G}^T$:
\begin{equation}
\widehat{\mathbf{D}}=\text{diag}\Big(\mathbf{G}\big(\mathbf{G}^T\mathbf{1}_n\big)\Big)\in\mathbb{R}^{n\times n},
\end{equation}
which shows that we can compute $\widehat{\mathbf{D}}$ by two matrix-vector multiplications. We thus  approximate the modified kernel matrix using the two matrices $\mathbf{G}$ and $\widehat{\mathbf{D}}$ in the following form:
\begin{equation}
\mathbf{M}\approx \widetilde{\mathbf{G}}\widetilde{\mathbf{G}}^T,\;\widetilde{\mathbf{G}}:=\widehat{\mathbf{D}}^{-1/2}\mathbf{G}\in\mathbb{R}^{n\times l}.
\end{equation}
The last step involves computing the $k$ leading left singular vectors of $\widetilde{\mathbf{G}}$ due to the following relationship between the SVD of $\widetilde{\mathbf{G}}=\mathbf{U}_{\widetilde{\mathbf{G}}}\boldsymbol{\Sigma}_{\widetilde{\mathbf{G}}}\mathbf{V}_{\widetilde{\mathbf{G}}}^T$ and the EVD of $\widetilde{\mathbf{G}}\widetilde{\mathbf{G}}^T$:
\begin{equation}
\widetilde{\mathbf{G}}\widetilde{\mathbf{G}}^T=\mathbf{U}_{\widetilde{\mathbf{G}}}\boldsymbol{\Sigma}_{\widetilde{\mathbf{G}}}^2\mathbf{U}_{\widetilde{\mathbf{G}}}^T,\;\mathbf{U}_{\widetilde{\mathbf{G}}}^T\mathbf{U}_{\widetilde{\mathbf{G}}}=\mathbf{I}_l.
\end{equation}

The proposed spectral clustering algorithm is summarized in Alg.~\ref{alg:SCour}. The key advantage of our approach is that it requires  a single pass over the matrix $\mathbf{C}\in\mathbb{R}^{n\times m}$, which encodes similarities between all the input data points and the landmark set. Given $\mathbf{C}$ and $\mathbf{W}$, the dominant computational cost for forming the matrix $\mathbf{G}$ in \eqref{eq:G}  is $\mathcal{O}(nml)$  arithmetic operations. Moreover, unlike the previous approach in \cite{li2011time}, we do not enforce a strict rank-$k$ approximation of the inner matrix $\mathbf{W}$. Instead, we exploit decay in the spectrum of $\mathbf{W}$, and the cost of forming $\mathbf{G}$ scales linearly with the underlying structure of the matrix $\mathbf{W}$. Also, it takes linear time in $n$ to compute the $k$ leading left singular vectors of $\widetilde{\mathbf{G}}$. Hence, the parameter $\gamma$ provides tunable trade-offs between efficiency and accuracy of the produced spectral embedding for large data sets. 
While we assumed that the computational cost of steps that do not involve $n$ is not a scalability barrier, one can utilize fast solvers for accelerating the computation of the spectral decomposition of the inner matrix $\mathbf{W}$ in the first step of our approach. For example, the previous work \cite{li2014large} employed randomized low-rank matrix approximation algorithms to compute the spectral decomposition of $\mathbf{W}$ in the Nystr\"om method. Thus, we can apply the same strategy to further accelerate our proposed spectral clustering algorithm.

\begin{algorithm}[t]
	\textbf{Input:}  data set $\mathcal{X}=\{\mathbf{x}_1,\ldots,\mathbf{x}_n\}$, landmark set $\mathcal{Z}=\{\mathbf{z}_1,\ldots,\mathbf{z}_m\}$, kernel parameter $\sigma$, number of clusters $k$, threshold parameter $\gamma$. 
	\begin{algorithmic}[1]
		\Function{SCNys}{$\mathcal{X},\mathcal{Z}, \sigma, k, \gamma$}
		\State Form $\mathbf{C}\in\mathbb{R}^{n\times m}$ and $\mathbf{W}\in\mathbb{R}^{m\times m}$, where $[\mathbf{C}]_{ij}=\kappa(\mathbf{x}_i,\mathbf{z}_j)$ and $[\mathbf{W}]_{ij}=\kappa(\mathbf{z}_i,\mathbf{z}_j)$;
		\State Compute the exact (or approximate) EVD of $\mathbf{W}=\mathbf{U}_{\mathbf{W}}\boldsymbol{\Sigma}_{\mathbf{W}}\mathbf{U}_{\mathbf{W}}^T$ and find the value of $l$ as in \eqref{eq:determinel};
		\State Construct $\mathbf{G}=\mathbf{C}\mathbf{U}_{\mathbf{W},l}\boldsymbol{\Sigma}_{\mathbf{W},l}^{-1/2}\in\mathbb{R}^{n\times l}$;
		\State Compute the degree matrix $\widehat{\mathbf{D}}=\text{diag}\Big(\mathbf{G}\big(\mathbf{G}^T\mathbf{1}_n\big)\Big)$;
		\State Find the $k$ leading left singular vectors of $\widetilde{\mathbf{G}}=\widehat{\mathbf{D}}^{-1/2}\mathbf{G}$ to obtain $\widehat{\mathbf{U}}_{\mathbf{M}}^{nys}=\mathbf{U}_{\widetilde{\mathbf{G}},k}\in\mathbb{R}^{n\times k}$;
		\State Normalize each row of $\widehat{\mathbf{U}}_{\mathbf{M}}^{nys}$ to have unit length;
		\State Perform K-means clustering over the rows of the normalized matrix $\widehat{\mathbf{U}}_{\mathbf{M}}^{nys}$;
		
		{
			\Return Clustering results.
		}
		\caption{Proposed Spectral Clustering Algorithm}\label{alg:SCour}
		\EndFunction
	\end{algorithmic}
\end{algorithm}

\subsection{Theoretical Results}
The goal of this section is to present theoretical insights and develop a better understanding of the role of the two low-rank approximations that we leveraged in our Nystr\"om-based spectral clustering algorithm. As a reference, the modified kernel matrix of the original data that we introduced in \eqref{eq:M} is $\mathbf{M}=\mathbf{D}^{-1/2}\mathbf{K}\mathbf{D}^{-1/2}$, and our method substitutes the kernel matrix $\mathbf{K}$ with the Nystr\"om approximation $\widehat{\mathbf{K}}=\mathbf{C}\mathbf{W}^\dagger\mathbf{C}^T$. While various theoretical guarantees exist on the quality of the Nystr\"om approximation, e.g., upper bounds for $\|\mathbf{K}-\widehat{\mathbf{K}}\|$ with respect to unitarily invariant matrix norms \cite{gittens2016revisiting}, a serious concern is that utilizing this approximation may have an out-sized influence on the normalized kernel matrix $\mathbf{M}$ due to the perturbation of the degree matrix.  The lack of theoretical understanding regarding the approximation quality of $\mathbf{M}$ is a critical disadvantage of Nystr\"om-based spectral clustering.

Moreover, the first step of our approach exploits the rank-$l$ approximation of the inner matrix $\mathbf{W}$. While the  incurred error $\|\mathbf{W}-\llbracket\mathbf{W}\rrbracket_l\|_2$ is negligible for small values of $\gamma$, this approximation does not necessarily guarantee that the error term $\|\widehat{\mathbf{K}}-\mathbf{C}\llbracket\mathbf{W}\rrbracket_l^\dagger\mathbf{C}^T\|_2$ is small. To explain this point, note that for any invertible matrices $\mathbf{A}$ and $\widehat{\mathbf{A}}$, we have $\widehat{\mathbf{A}}^{-1}-\mathbf{A}^{-1}=\widehat{\mathbf{A}}^{-1}(\mathbf{A}-\widehat{\mathbf{A}})\mathbf{A}^{-1}$. Thus, the small norm of $(\mathbf{A}-\widehat{\mathbf{A}})$ cannot be directly used to conclude their inverses are close to each other or even bounded. In this section, we start with providing theoretical guarantees for understanding the influence of the rank-$l$ approximation of $\mathbf{W}$. 

\begin{theorem}[Rank-$l$ approximation of $\mathbf{W}$] Consider a set of $n$ distinct data points $\mathcal{X}=\{\mathbf{x}_1,\ldots,\mathbf{x}_n\}$ and a landmark set $\mathcal{Z}=\{\mathbf{z}_1,\ldots,\mathbf{z}_m\}$ sampled from $\mathcal{X}$ using a set of indices $\mathcal{I}\subseteq\{1,\ldots,n\}$. Let $\mathbf{K}\in\mathbb{R}^{n\times n}$ denote the kernel matrix associated with $\mathcal{X}$ using the Gaussian kernel function $\kappa(\mathbf{x}_i,\mathbf{x}_j)=\exp(-\|\mathbf{x}_i-\mathbf{x}_j\|_2^2/\sigma^2)$. Also, let $\mathbf{P}\in\mathbb{R}^{n\times m}$ be a subset of the columns of $\mathbf{I}_n$ selected according to the set of indices $\mathcal{I}$. The Nystr\"om method computes $\mathbf{C}=\mathbf{K}\mathbf{P}\in\mathbb{R}^{n\times m}$ and $\mathbf{W}=\mathbf{P}^T\mathbf{K}\mathbf{P}\in\mathbb{R}^{m\times m}$ to form $\mathbf{C}\mathbf{W}^\dagger\mathbf{C}^T$. The error incurred by the best rank-$l$ approximation of the inner matrix $\mathbf{W}$  can be characterized as follows for any $l < m$:
\begin{equation}
\mathbf{C}\big(\mathbf{W}^\dagger-\llbracket\mathbf{W}\rrbracket_l^\dagger\big)\mathbf{C}^T=\mathbf{K}^{1/2}\big(\mathbf{U}_{\mathbf{F}}\mathbf{U}_{\mathbf{F}}^T-\mathbf{U}_{\mathbf{F},l}\mathbf{U}_{\mathbf{F},l}^T\big)\mathbf{K}^{1/2},\label{eq:thm1}
\end{equation}
where $\mathbf{U}_{\mathbf{F}}\in\mathbb{R}^{n\times m}$ represents the left singular vectors of $\mathbf{F}:=\mathbf{K}^{1/2}\mathbf{P}$ and  $\mathbf{U}_{\mathbf{F},l}\in\mathbb{R}^{n\times l}$ denotes its first $l$ columns. 
\end{theorem}
\begin{proof}
Let us start with the truncated SVD of $\mathbf{F}=\mathbf{U}_{\mathbf{F}}\boldsymbol{\Sigma}_{\mathbf{F}}\mathbf{V}_{\mathbf{F}}^T$, which allows us to rewrite $\mathbf{C}$ as follows:
\begin{equation}
\mathbf{C}=\mathbf{K}^{1/2}(\mathbf{K}^{1/2}\mathbf{P})=\mathbf{K}^{1/2}\mathbf{F}=\mathbf{K}^{1/2}\mathbf{U}_{\mathbf{F}}\boldsymbol{\Sigma}_{\mathbf{F}}\mathbf{V}_{\mathbf{F}}^T,\label{eq:CF}
\end{equation}
and we get the following representation for $\mathbf{W}$:
\begin{equation}
\mathbf{W}=(\mathbf{P}^T\mathbf{K}^{1/2})(\mathbf{K}^{1/2}\mathbf{P})=\mathbf{F}^T\mathbf{F}=\mathbf{V}_{\mathbf{F}}\boldsymbol{\Sigma}_{\mathbf{F}}^2\mathbf{V}_{\mathbf{F}}^T.\label{eq:WF}
\end{equation}
Using \eqref{eq:CF} and \eqref{eq:WF}, it is straightforward to show that:
\begin{equation}
\mathbf{C}\mathbf{W}^\dagger\mathbf{C}^T=\mathbf{K}^{1/2}\mathbf{U}_{\mathbf{F}}\mathbf{U}_{\mathbf{F}}^T\mathbf{K}^{1/2},
\end{equation}
where we used $\mathbf{V}_{\mathbf{F}}^T\mathbf{V}_{\mathbf{F}}=\mathbf{V}_{\mathbf{F}}\mathbf{V}_{\mathbf{F}}^T=\mathbf{I}_m$ since the kernel matrix $\mathbf{W}$ associated with $\mathcal{Z}$ is full-rank when employing the Gaussian kernel function on a set of distinct data points. 

Next, we use the best rank-$l$ approximation of $\mathbf{W}$, i.e., $\llbracket\mathbf{W}\rrbracket_l=\mathbf{V}_{\mathbf{F},l}\boldsymbol{\Sigma}_{\mathbf{F},l}^2\mathbf{V}_{\mathbf{F},l}^T$, to simplify the Nystr\"om approximation after replacing $\mathbf{W}$ with $\llbracket\mathbf{W}\rrbracket_l$ as follows:
\begin{equation}
\mathbf{C}\llbracket\mathbf{W}\rrbracket_l^\dagger\mathbf{C}^T=\mathbf{K}^{1/2}\mathbf{U}_\mathbf{F}\mathbf{T}\mathbf{U}_\mathbf{F}^T\mathbf{K}^{1/2},
\end{equation}
where $\mathbf{T}:=\boldsymbol{\Sigma}_{\mathbf{F}}\mathbf{V}_{\mathbf{F}}^T\mathbf{V}_{\mathbf{F},l}\boldsymbol{\Sigma}_{\mathbf{F},l}^{-2}\mathbf{V}_{\mathbf{F},l}^T\mathbf{V}_{\mathbf{F}}\boldsymbol{\Sigma}_{\mathbf{F}}\in\mathbb{R}^{m\times m}$. Note that the right singular vectors of $\mathbf{F}$ can be decomposed as:
\begin{equation}
\mathbf{V}_{\mathbf{F},l}^T\mathbf{V}_{\mathbf{F}}=\begin{bmatrix}\mathbf{V}_{\mathbf{F},l}^T\mathbf{V}_{\mathbf{F},l} &\mathbf{V}_{\mathbf{F},l}^T\mathbf{V}_{\mathbf{F},l}^\perp \end{bmatrix}=\begin{bmatrix}\mathbf{I}_l&\mathbf{0}_{l\times (m-l)}\end{bmatrix},
\end{equation}
where $\mathbf{V}_{\mathbf{F},l}^\perp\in\mathbb{R}^{m\times (m-l)}$ represents the tailing $(m-l)$ right singular vectors of $\mathbf{F}$ and we used the fact that the columns of $\mathbf{V}_{\mathbf{F}}$ are orthogonal. Hence, we see that $\mathbf{T}$ has a block structure:
\begin{equation}
\mathbf{T}=\begin{bmatrix}\mathbf{I}_{l\times l} & \mathbf{0}_{l\times (m-l)}\\\mathbf{0}_{(m-l)\times l}&\mathbf{0}_{(m-l)\times (m-l)}\end{bmatrix}.
\end{equation}
As a result, we have $\mathbf{C}\llbracket\mathbf{W}\rrbracket_l^\dagger\mathbf{C}^T=\mathbf{K}^{1/2}\mathbf{U}_{\mathbf{F},l}\mathbf{U}_{\mathbf{F},l}^T\mathbf{K}^{1/2}$,  which completes the proof. 
\end{proof}

The presented result in \eqref{eq:thm1} allows us to develop a better understanding of utilizing the rank-$l$ approximation of $\mathbf{W}$ via the SVD of $\mathbf{F}$ without any matrix inversion. Based on this result and using standard inequalities for the spectral norm, we conclude the error incurred by the low-rank approximation is always bounded for any value of $l\in\{1,2,\ldots,m\}$: 
\begin{equation}
\|\mathbf{C}\big(\mathbf{W}^\dagger-\llbracket\mathbf{W}\rrbracket_l^\dagger\big)\mathbf{C}^T\|_2\leq\underbrace{\|\mathbf{K}^{1/2}\|_2^2}_{=\|\mathbf{K}\|_2}\underbrace{\|\mathbf{U}_{\mathbf{F}}\mathbf{U}_{\mathbf{F}}^T-\mathbf{U}_{\mathbf{F},l}\mathbf{U}_{\mathbf{F},l}^T\|_2}_{=1},
\end{equation}
which can be simplified as:
\begin{equation}
\frac{\|\mathbf{C}\big(\mathbf{W}^\dagger-\llbracket\mathbf{W}\rrbracket_l^\dagger\big)\mathbf{C}^T\|_2}{\|\mathbf{K}\|_2}\leq 1.\label{eq:normalizedE}
\end{equation}
A significant advantage of this upper bound compared to prior results, such as the bound presented in \cite{drineas2005nystrom}, is that our theoretical analysis does not make any assumptions on the landmark selection process involved in the Nystr\"om method.  

\begin{figure}[t]
	\centering
	\includegraphics[width=\linewidth]{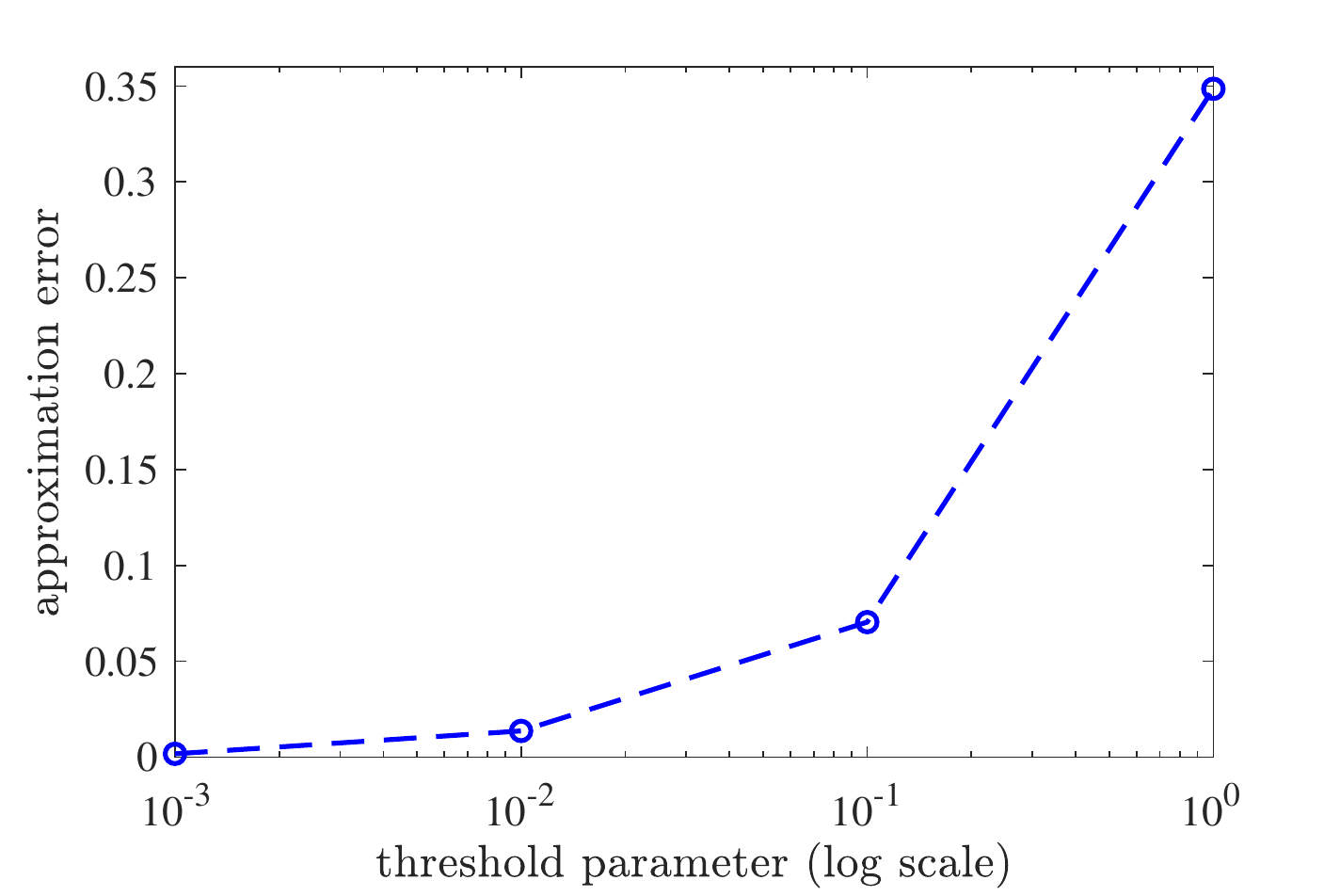}
	
	\caption{Investigating the influence of the threshold parameter $\gamma$ on the normalized approximation error for fixed $m=200$.}
	\label{fig:them1} 
\end{figure}

We provide numerical evidence measuring the normalized approximation error in \eqref{eq:normalizedE} using the mushrooms data set that we considered in Section \ref{sec:rel}.  Fig.~\ref{fig:them1} reports the mean approximation error over $50$ independent trials with fixed $m=200$ and varying values of the threshold parameter $\gamma\in\{10^{-3},10^{-2},10^{-1},10^0\}$. As discussed before, the normalized approximation error is less than $1$ even when we set the threshold parameter $\gamma=1$, yielding the rank-$1$ approximation of $\mathbf{W}$ under the assumption that $\sigma_1(\mathbf{W})>\sigma_2(\mathbf{W})$. In practice, one has the flexibility to reduce the value of $\gamma$ to lower the resulting rank-$l$ approximation error, achieving a trade-off between accuracy and scalability. As we increase $\gamma$ in this experiment, the mean values of the rank parameter $l$ are $196.6$, $76.6$, $6.2$, and $1$, respectively.

Our second theoretical analysis focuses on developing an upper bound for the spectral norm of the perturbation on the modified kernel matrix $\mathbf{M}$ when utilizing the Nystr\"om approximation. Therefore, our theoretical result reduces the current gap between a high-quality approximation of the kernel matrix $\mathbf{K}$ to a provably reasonable estimate of $\mathbf{M}$. Since we have the following relationship between the modified kernel matrix and the graph Laplacian matrix $\mathbf{L}=\mathbf{I}_n-\mathbf{M}$, our result immediately translates to the perturbation analysis of the graph Laplacian under the Nystr\"om approximation $\widehat{\mathbf{K}}=\mathbf{C}\mathbf{W}^\dagger\mathbf{C}^T$.  
\begin{theorem}[Low-rank approximation of $\mathbf{K}$\label{thm2}] Consider the kernel matrix $\mathbf{K}\in\mathbb{R}^{n\times n}$ representing pairwise similarities and the modified kernel matrix $\mathbf{M}=\mathbf{D}^{-1/2}\mathbf{K}\mathbf{D}^{-1/2}$, where $\mathbf{D}=\text{diag}(\mathbf{K}\mathbf{1}_n)$ is the diagonal degree matrix. Let $\widehat{\mathbf{K}}=\mathbf{K}+\mathbf{E}$ and $\widehat{\mathbf{D}}=\mathbf{D}+\boldsymbol{\Delta}$ denote the perturbed kernel matrix and its diagonal degree matrix, respectively. If the amount of perturbation is small in the sense that $\eta:=\|\boldsymbol{\Delta}\mathbf{D}^{-1}\|_2<1$, then the normalized difference between $\widehat{\mathbf{M}}:=\widehat{\mathbf{D}}^{-1/2}\widehat{\mathbf{K}}\widehat{\mathbf{D}}^{-1/2}$ and $\mathbf{M}$ is bounded in the spectral norm as follows:
\begin{equation}
\frac{\|\mathbf{M}-\widehat{\mathbf{M}}\|_2}{\|\mathbf{M}\|_2}\leq f_1+f_2,\label{eq:thm2up}
\end{equation}
where 
\begin{equation}
f_1:= (1+\eta+\mathcal{O}(\eta^2))\frac{\|\mathbf{D}^{-1/2}\mathbf{E}\mathbf{D}^{-1/2}\|_2}{\|\mathbf{M}\|_2},
\end{equation}
and 
\begin{equation}
f_2:=\eta+\mathcal{O}(\eta^2).
\end{equation}	
\end{theorem}
\begin{proof}
Using $\widehat{\mathbf{K}}$ and $\widehat{\mathbf{D}}$, we express the perturbed modified kernel matrix $\widehat{\mathbf{M}}$ in the following form:
\begin{equation}
(\mathbf{I}_n+\boldsymbol{\Delta}\mathbf{D}^{-1})^{-1/2}\mathbf{D}^{-1/2}(\mathbf{K}+\mathbf{E})\mathbf{D}^{-1/2}(\mathbf{I}_n+\boldsymbol{\Delta}\mathbf{D}^{-1})^{-1/2},\label{eq:Mhat}
\end{equation}
where we used the following relationship: 
\begin{equation}
\mathbf{D}+\boldsymbol{\Delta}=\mathbf{D}(\mathbf{I}_n+\mathbf{D}^{-1}\boldsymbol{\Delta})=(\mathbf{I}_n+\boldsymbol{\Delta}\mathbf{D}^{-1})\mathbf{D}. 
\end{equation}
and  $\mathbf{D}^{-1}\boldsymbol{\Delta}=\boldsymbol{\Delta}\mathbf{D}^{-1}$ because both $\mathbf{D}$ and $\boldsymbol{\Delta}$ are diagonal matrices. Also, for any invertible matrices $\mathbf{A}$ and $\mathbf{B}$, the matrix product is invertible and we have $(\mathbf{A}\mathbf{B})^{-1}=\mathbf{B}^{-1}\mathbf{A}^{-1}$.

Next, we use the Taylor series for matrix functions under the assumption that the perturbation amount on the degree matrix satisfies $\eta=\|\boldsymbol{\Delta}\mathbf{D}^{-1}\|_2<1$, which yields \cite{yan2009fast}:
\begin{equation}
(\mathbf{I}+\boldsymbol{\Delta}\mathbf{D}^{-1})^{-1/2}=\mathbf{I}_n-\frac{1}{2}\boldsymbol{\Delta}\mathbf{D}^{-1}+\mathcal{O}((\boldsymbol{\Delta}\mathbf{D}^{-1})^2).\label{eq:taylor}
\end{equation}
We then substitute \eqref{eq:taylor} in the expression for the perturbed modified kernel matrix $\widehat{\mathbf{M}}$ given in \eqref{eq:Mhat}. It is straightforward to show that $\widehat{\mathbf{M}}=g_1+g_2+g_3$, where we have:
\begin{multline}
g_1=\mathbf{M}+\widetilde{\mathbf{E}}-\frac{1}{2}\mathbf{M}(\boldsymbol{\Delta}\mathbf{D}^{-1})-\frac{1}{2}\widetilde{\mathbf{E}}(\boldsymbol{\Delta}\mathbf{D}^{-1})\\
+\mathbf{M}\mathcal{O}((\boldsymbol{\Delta}\mathbf{D}^{-1})^2)+\widetilde{\mathbf{E}}\mathcal{O}((\boldsymbol{\Delta}\mathbf{D}^{-1})^2),
\end{multline}
and we introduced $\widetilde{\mathbf{E}}:=\mathbf{D}^{-1/2}\mathbf{E}\mathbf{D}^{-1/2}$. Similarly, we get:
\begin{multline}
g_2=-\frac{1}{2}(\boldsymbol{\Delta}\mathbf{D}^{-1})\mathbf{M}-\frac{1}{2}(\boldsymbol{\Delta}\mathbf{D}^{-1})\widetilde{\mathbf{E}}+\frac{1}{4}(\boldsymbol{\Delta}\mathbf{D}^{-1})\mathbf{M}(\boldsymbol{\Delta}\mathbf{D}^{-1})\\
+\frac{1}{4}(\boldsymbol{\Delta}\mathbf{D}^{-1})\widetilde{\mathbf{E}}(\boldsymbol{\Delta}\mathbf{D}^{-1})-\frac{1}{2}(\boldsymbol{\Delta}\mathbf{D}^{-1})\mathbf{M}\mathcal{O}((\boldsymbol{\Delta}\mathbf{D}^{-1})^2)\\-\frac{1}{2}(\boldsymbol{\Delta}\mathbf{D}^{-1})\widetilde{\mathbf{E}}\mathcal{O}((\boldsymbol{\Delta}\mathbf{D}^{-1})^2),
\end{multline}
and we find the third term:
\begin{multline}
g_3=\mathcal{O}((\boldsymbol{\Delta}\mathbf{D}^{-1})^2)\mathbf{M}+\mathcal{O}((\boldsymbol{\Delta}\mathbf{D}^{-1})^2)\widetilde{\mathbf{E}}\\
-\frac{1}{2}\mathcal{O}((\boldsymbol{\Delta}\mathbf{D}^{-1})^2)\mathbf{M}(\boldsymbol{\Delta}\mathbf{D}^{-1})-\frac{1}{2}\mathcal{O}((\boldsymbol{\Delta}\mathbf{D}^{-1})^2)\widetilde{\mathbf{E}}(\boldsymbol{\Delta}\mathbf{D}^{-1})\\
+\mathcal{O}((\boldsymbol{\Delta}\mathbf{D}^{-1})^2)\mathbf{M}\mathcal{O}((\boldsymbol{\Delta}\mathbf{D}^{-1})^2)\\+\mathcal{O}((\boldsymbol{\Delta}\mathbf{D}^{-1})^2)\widetilde{\mathbf{E}}\mathcal{O}((\boldsymbol{\Delta}\mathbf{D}^{-1})^2).
\end{multline}
Therefore, we compute $\mathbf{M}-\widehat{\mathbf{M}}$ and use standard inequalities for the spectral norm to find the following upper bound:
\begin{multline}
\|\mathbf{M}-\widehat{\mathbf{M}}\|_2\leq \|\widetilde{\mathbf{E}}\|_2+\eta\|\widetilde{\mathbf{E}}\|_2+\mathcal{O}(\eta^2)\|\widetilde{\mathbf{E}}\|_2\\+\eta\|\mathbf{M}\|_2+\mathcal{O}(\eta^2)\|\mathbf{M}\|_2,
\end{multline}
and we divide both sides by $\|\mathbf{M}\|_2$ to complete the proof. 
\end{proof}

This theorem provides an upper bound for the spectral norm of the perturbation on the modified kernel matrix $\mathbf{M}$. This result also applies to the perturbation analysis of the normalized graph Laplacian $\mathbf{L}$. The main underlying assumption is that the term  $\eta=\|\boldsymbol{\Delta}\mathbf{D}^{-1}\|_2$ should be less than $1$, where both $\boldsymbol{\Delta}$ and $\mathbf{D}$ are diagonal matrices. That is, the absolute error incurred by estimating the degree of each node when utilizing the perturbed kernel matrix of the similarity graph should be less than the actual degree of the corresponding node. To better understand the influence of utilizing an approximate kernel matrix on the degree matrix, we have the following connection between $\mathbf{E}$ and $\boldsymbol{\Delta}$: 
\begin{equation}
\widehat{\mathbf{D}}=\text{diag}(\widehat{\mathbf{K}}\mathbf{1}_n)=\underbrace{\text{diag}(\mathbf{K}\mathbf{1}_n)}_{=\mathbf{D}}+\underbrace{\text{diag}(\mathbf{E}\mathbf{1}_n)}_{=\boldsymbol{\Delta}}.
\end{equation}
Next, we  find an upper bound for the spectral norm of the diagonal matrix $\boldsymbol{\Delta}$ as follows:
\begin{equation}
\|\boldsymbol{\Delta}\|_2=\|\mathbf{E}\mathbf{1}_n\|_{\infty}\leq \|\mathbf{E}\mathbf{1}_n\|_2\leq \sqrt{n}\|\mathbf{E}\|_2,
\end{equation}
where $\|\mathbf{x}\|_{\infty}:=\max_{i}|x_i|$ represents the infinity norm for vectors and we used the following inequality for vector norms $\|\mathbf{x}\|_{\infty}\leq \|\mathbf{x}\|_2$. Therefore, the error associated with estimating the degree matrix depends on the kernel matrix approximation error. Furthermore, according to \eqref{eq:thm2up}, we find out the upper bound depends on the quality of the Nystr\"om approximation, i.e., $\|\mathbf{E}\|_2=\|\mathbf{K}-\mathbf{C}\mathbf{W}^\dagger\mathbf{C}^T\|_2$. As mentioned before, this error term is known to be bounded in the Nystr\"om method for unitarily invariant norms and various landmark selection strategies such as uniform sampling without replacement \cite{kumar2012sampling}.

Theorem \ref{thm2} allows us to show that the distance between subspaces spanned by the $k$ leading eigenvectors of $\mathbf{M}$ and $\widehat{\mathbf{M}}$ is bounded. The well-known Davis-Kahan theorem \cite{davis1970rotation} of matrix perturbation theory  states that the sine of the largest principal angle between the two subspaces is upper bounded by the perturbation amount: $\|\sin \boldsymbol{\Theta}(\mathbf{U}_{\mathbf{M},k},\mathbf{U}_{\widehat{\mathbf{M}},k})\|_2\le \|\mathbf{M}-\widehat{\mathbf{M}}\|_2/\delta$, where $\delta>0$ is an eigenvalue separation parameter \cite{yu2015useful,eldridge2018unperturbed}. Therefore, the quality of the Nystr\"om-based spectral embedding improves by reducing the kernel matrix approximation error $\|\mathbf{E}\|_2=\|\mathbf{K}-\mathbf{C}\mathbf{W}^\dagger\mathbf{C}^T\|_2$. 

\begin{figure*}[t]
	\centering
	\subfloat[][\label{fig:eta} $\eta=\|\boldsymbol{\Delta}\mathbf{D}^{-1}\|_2$]{
		\includegraphics[width=0.5\linewidth]{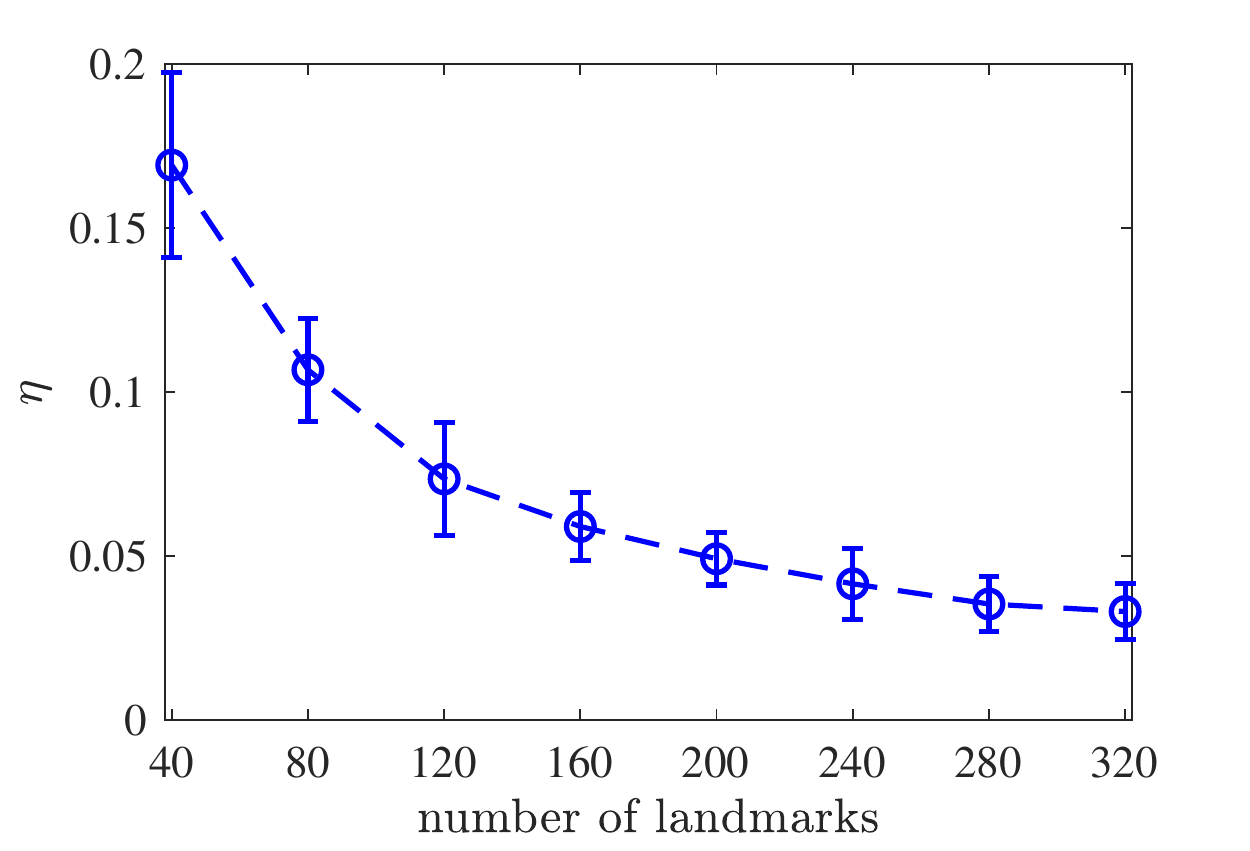}
	}
	\subfloat[][\label{fig:Merror} $\|\mathbf{M}-\widehat{\mathbf{M}}\|_2/\|\mathbf{M}\|_2$]{
		\includegraphics[width=0.5\linewidth]{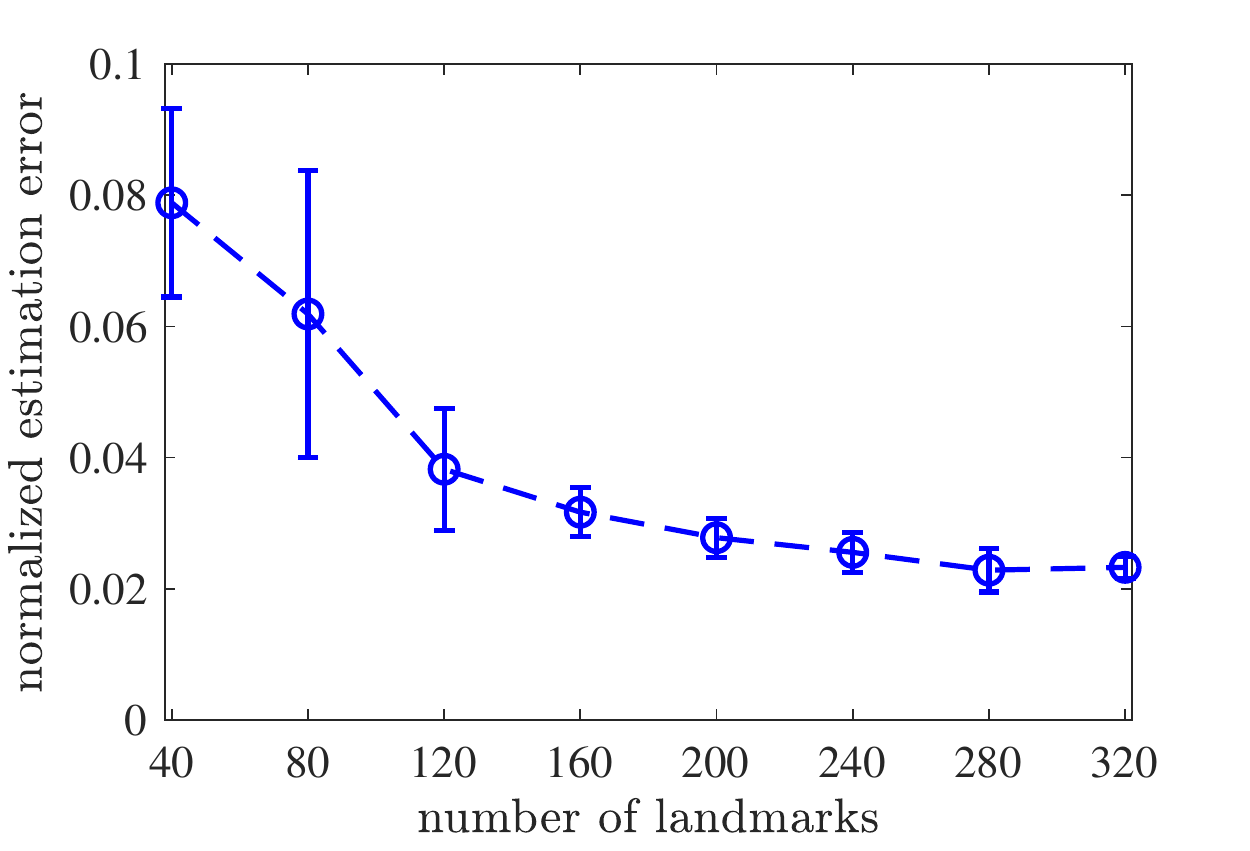}
	} 
	\caption{\label{fig:thm2}
	Empirical investigation of the main assumption and the upper bound for the normalized approximation error in Theorem \ref{thm2}.
	}
\end{figure*}

To further explain the assumption made in Theorem~\ref{thm2} and understand the upper bound for the normalized difference between $\mathbf{M}$ and $\widehat{\mathbf{M}}$, we revisit the numerical simulation performed on the mushrooms data set. For varying values of $m$, we report the average and standard deviation values of $\eta=\|\boldsymbol{\Delta}\mathbf{D}^{-1}\|_2$ and $\|\mathbf{M}-\widehat{\mathbf{M}}\|_2/\|\mathbf{M}\|_2$ over $50$ independent trials in Fig.~\ref{fig:eta} and \ref{fig:Merror}, respectively. We also set the threshold parameter $\gamma=10^{-2}$ in our proposed approach to consider the error introduced by the low-rank approximation of the inner matrix $\mathbf{W}$. As the number of landmarks $m$ increases, we observe a decreasing trend for both quantities because of obtaining more accurate Nystr\"om kernel matrix approximations. Note that even for small values of $m$, e.g., $m=40$, the underlying assumption of Theorem \ref{thm2} is satisfied since $\eta<1$. We also notice that the normalized approximation error for the modified kernel matrix $\mathbf{M}$ gets close to $0.02$ for $m=320$.

\section{Experimental Results}\label{sec:exper}
In this section, we conduct extensive experiments to assess the performance and time complexity of the proposed spectral clustering method on several synthetic and real data sets. We compare our proposed method with other competing techniques that circumvent the construction of full similarity graphs in large-scale settings. All the studied spectral clustering methods are implemented in Matlab. In our approach, Matlab built-in functions are used for computing standard matrix factorizations, including SVD and EVD. We also use Matlab's internal K-means in the last step of spectral clustering to partition the produced spectral embedding into $k$ clusters. For the K-means  algorithm in the last step of spectral clustering, we set the maximum number of iterations to 10 since increasing the number of iterations does not make any notable difference based on our evaluation. 

This section reports two evaluation metrics to assess the performance of spectral clustering techniques using ground-truth labels. Let $\mathcal{L}_1,\ldots,\mathcal{L}_k$ be the ground-truth partition of the data, and $\widehat{\mathcal{L}}_1,\ldots,\widehat{\mathcal{L}}_k$ represent the output of a spectral clustering algorithm. Thus, $n_{ij}=|\mathcal{L}_i\cap\widehat{\mathcal{L}}_j|$ denotes the number of shared samples in $\mathcal{L}_i$ and $\widehat{\mathcal{L}}_j$. Hence, $p_{ij}=n_{ij}/|\widehat{\mathcal{L}}_j|$ and $r_{ij}=n_{ij}/|\mathcal{L}_i|$ represent the precision and recall, respectively. Using these two quantities, the F-score between these two partitions is defined as $F_{ij}=(2p_{ij}r_{ij})/(p_{ij}+r_{ij})$ \cite{you2018scalable}. Let $\Pi$ be the set of all permutations of $\{1,\ldots,k\}$, the average F-score that we report in this paper has the following form:
\begin{equation}
\text{F-score}=\max_{\pi\in\Pi} \frac{1}{k}\sum_{i=1}^{k} F_{i\pi(i)},
\end{equation}
where we follow the common practice of permuting the obtained clusters to best match the ground-truth.  Another metric that we use for evaluation of spectral clustering methods is normalized mutual information (NMI) \cite{strehl2002cluster}, defined as:
\begin{equation}
\text{NMI}=\frac{2I(\mathcal{L};\widehat{\mathcal{L}})}{H(\mathcal{L})+H(\widehat{\mathcal{L}})},
\end{equation}
where $I(\mathcal{L};\widehat{\mathcal{L}})$ denotes the mutual information between the ground-truth and returned clusters, and $H(\cdot)$ represents the entropy of each partition. Both evaluation metrics take values between zero and one, with a larger score indicating better clustering performance. 

For all Nystr\"om-based spectral clustering algorithms, we employ uniform sampling without replacement to select the landmark set. 
 While landmark selection strategies abound in the literature, the primary purpose of this work is to present a principled spectral clustering method to optimize accuracy-efficiency trade-offs for a given Nystr\"om approximation (i.e., fixed $\mathbf{C}$ and $\mathbf{W}$). Towards this goal, we presented two theoretical results that hold for any landmark selection process and designing improved landmark selection techniques is out of the scope of this work. As shown in Theorem \ref{thm2}, a more accurate Nystr\"om approximation results in a smaller perturbation of the modified kernel matrix in prototypical spectral clustering. We thus decided to use uniform sampling for simplicity and  ease of implementation. Since the sampling step involves randomness, we repeat each experiment over $50$ independent trials and report the average values (along with standard deviations for cluster quality evaluations). For our approach, we set the threshold parameter $\gamma$, introduced in \eqref{eq:determinel} for the low-rank approximation of $\mathbf{W}$, to $10^{-2}$  unless otherwise specified. We will present a numerical simulation to investigate the sensitivity of our method to $\gamma$. This section starts with demonstrating the performance and efficiency of our method using synthetic data and we then focus on two real data sets, namely mushrooms and MNIST, from LIBSVM \cite{CC01a}. 
 
 \begin{figure*}[t]
 	\centering
 	\subfloat[][\label{fig:data-moons}  moons, $k=2$ clusters]{
 		\includegraphics[width=0.33\linewidth]{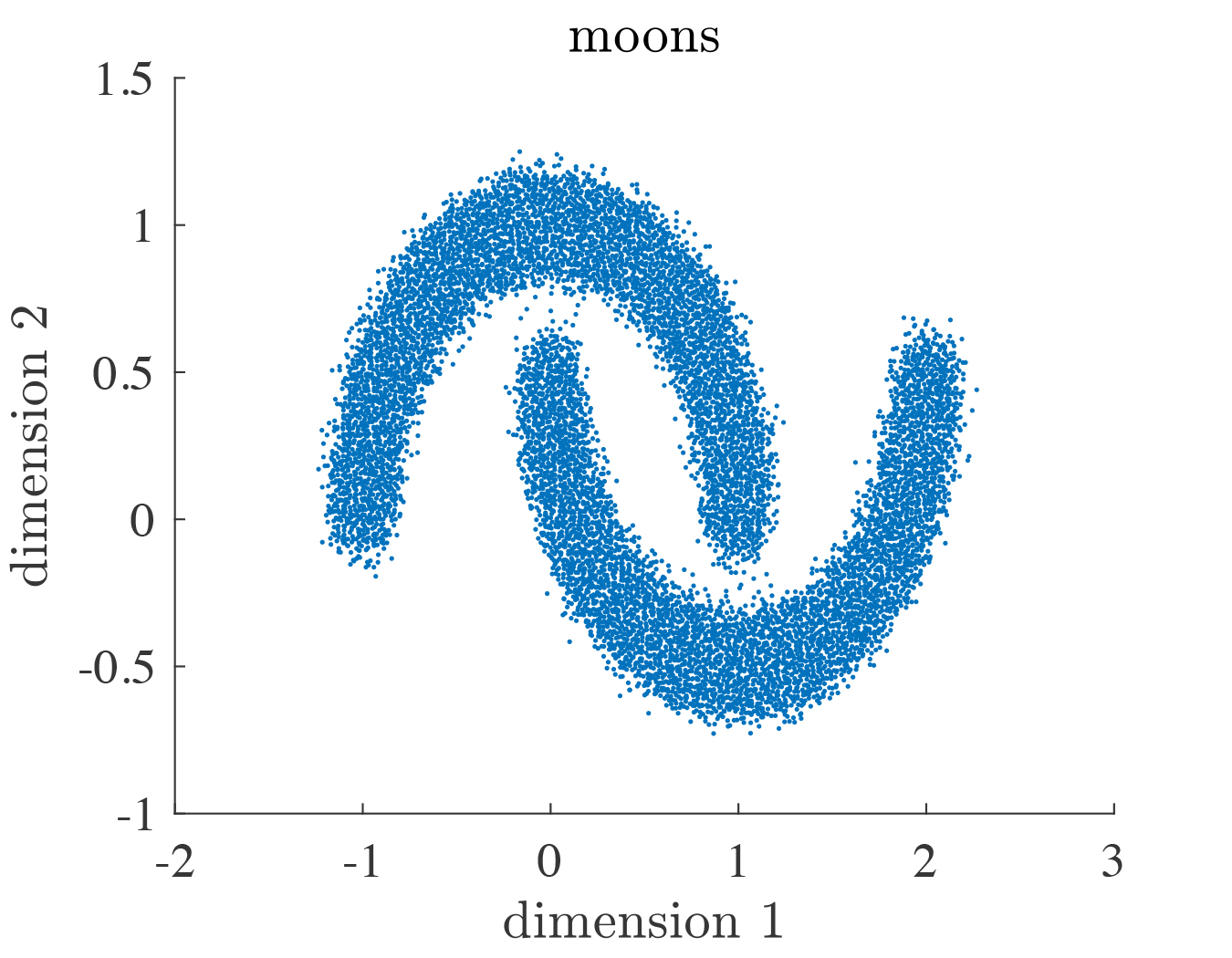}
 	}
 	\subfloat[][\label{fig:data-circles} circles, $k=2$ clusters]{
 		\includegraphics[width=0.33\linewidth]{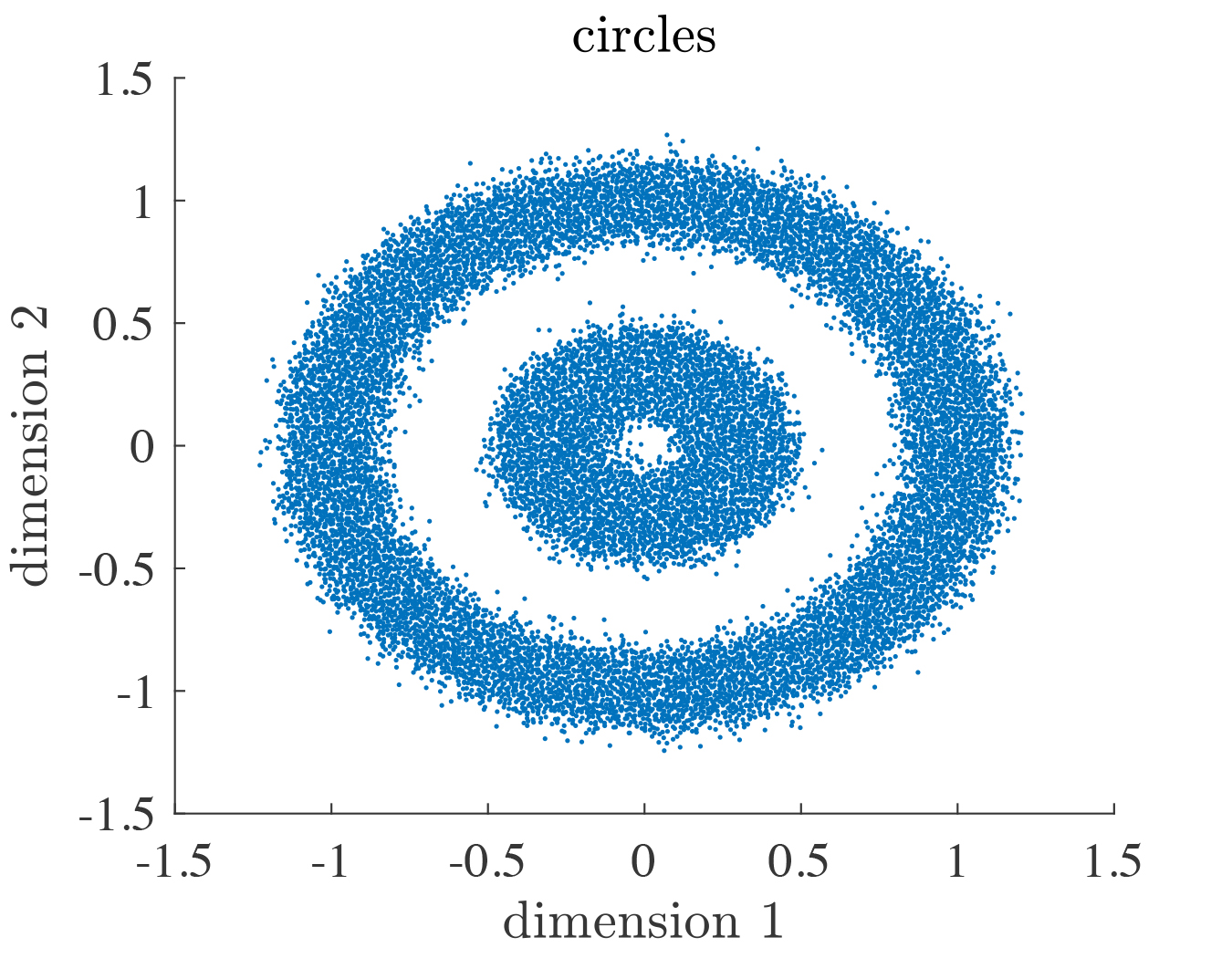}
 	}
 	\subfloat[][\label{fig:data-blobs} blobs, $k=3$ clusters]{
 		\includegraphics[width=0.33\linewidth]{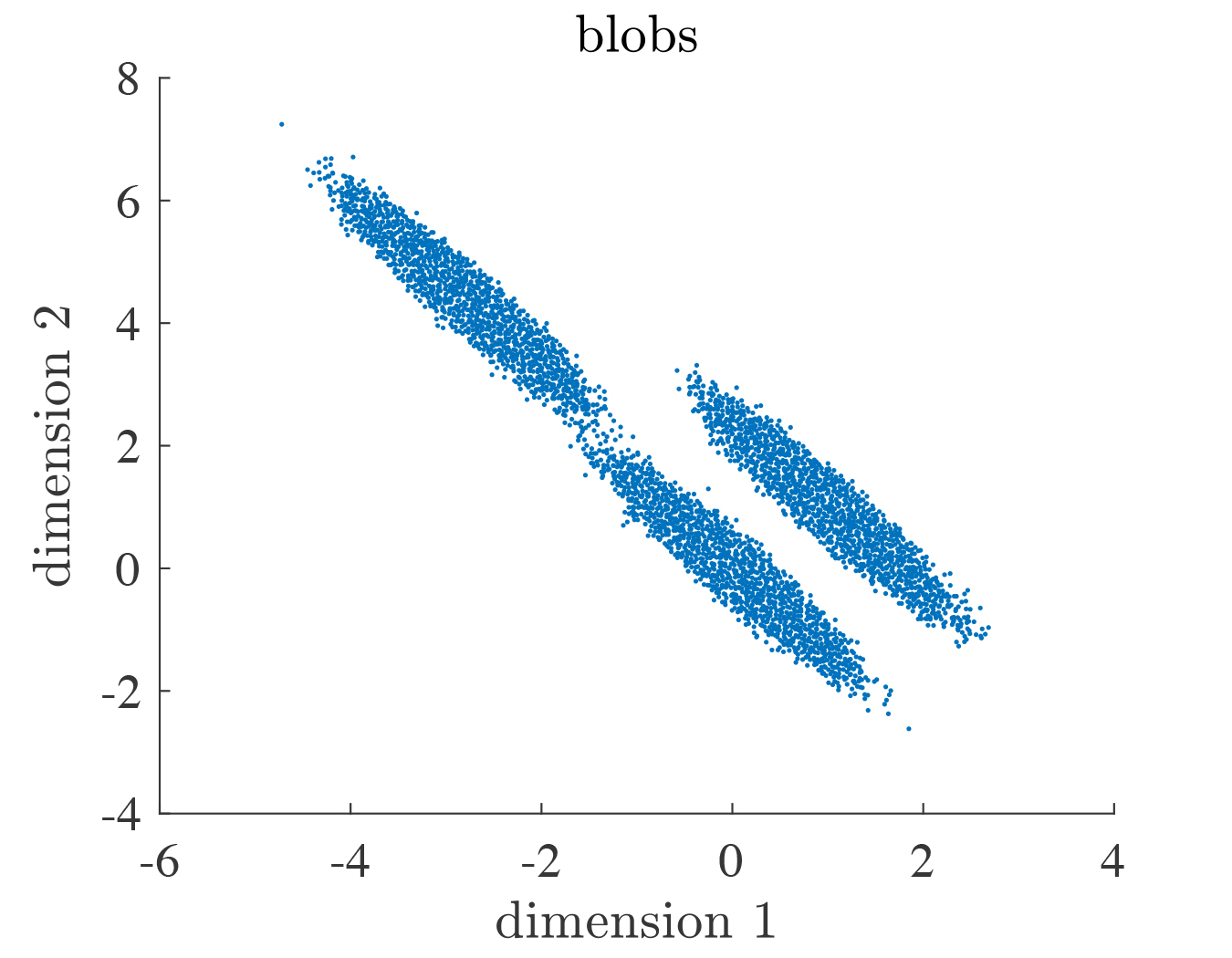}
 	}
 	
 	\subfloat[][\label{fig:fsc-moons}  moons, F-score]{
 		\includegraphics[width=0.33\linewidth]{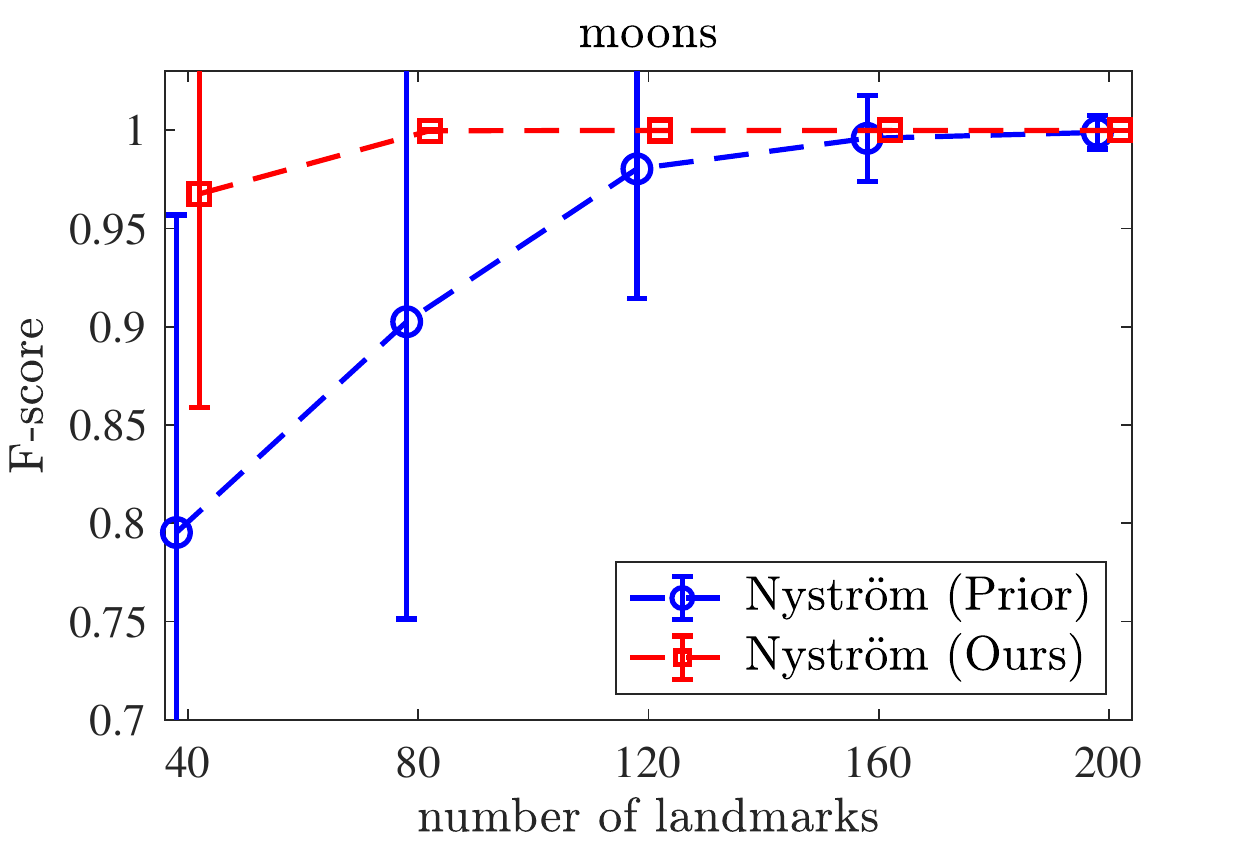}
 	}
 	\subfloat[][\label{fig:fsc-circles} circles, F-score]{
 		\includegraphics[width=0.33\linewidth]{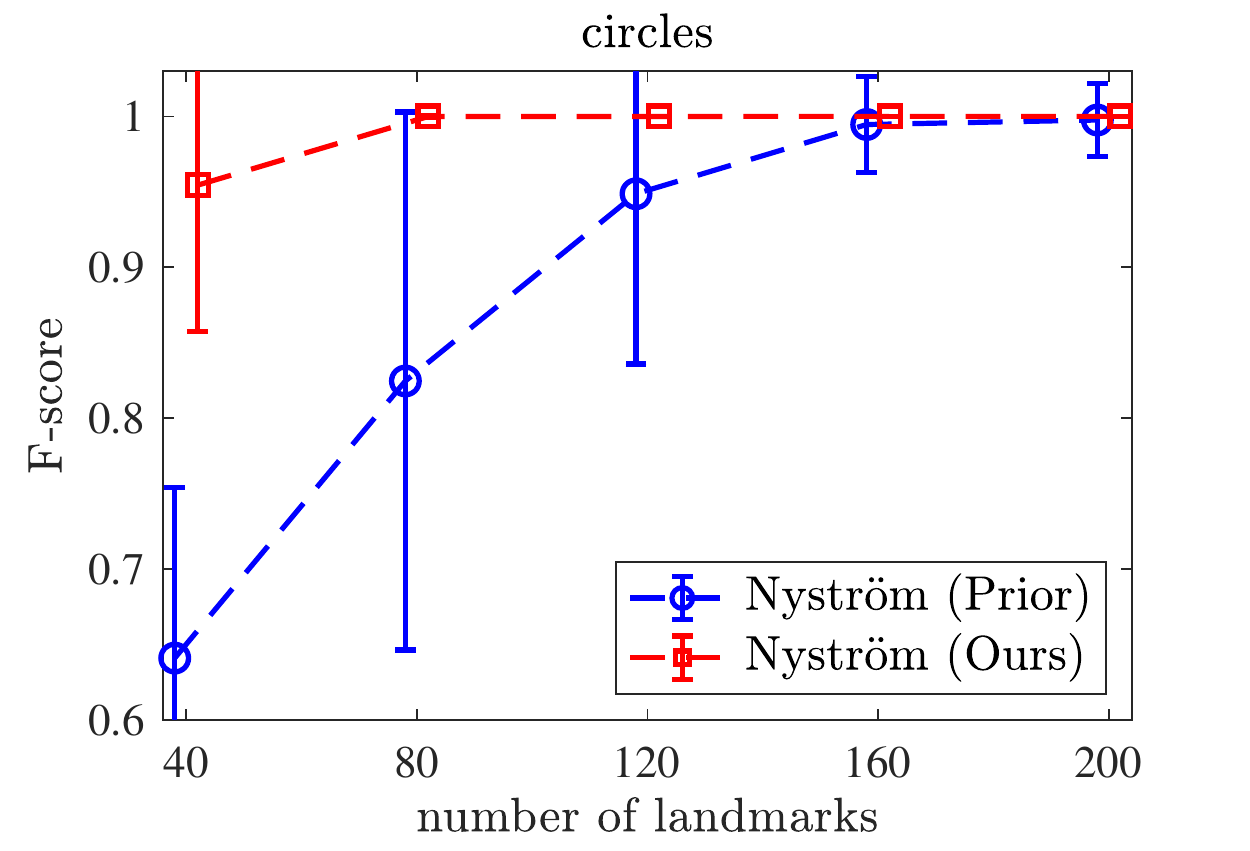}
 	}
 	\subfloat[][\label{fig:fsc-blobs} blobs, F-score]{
 		\includegraphics[width=0.33\linewidth]{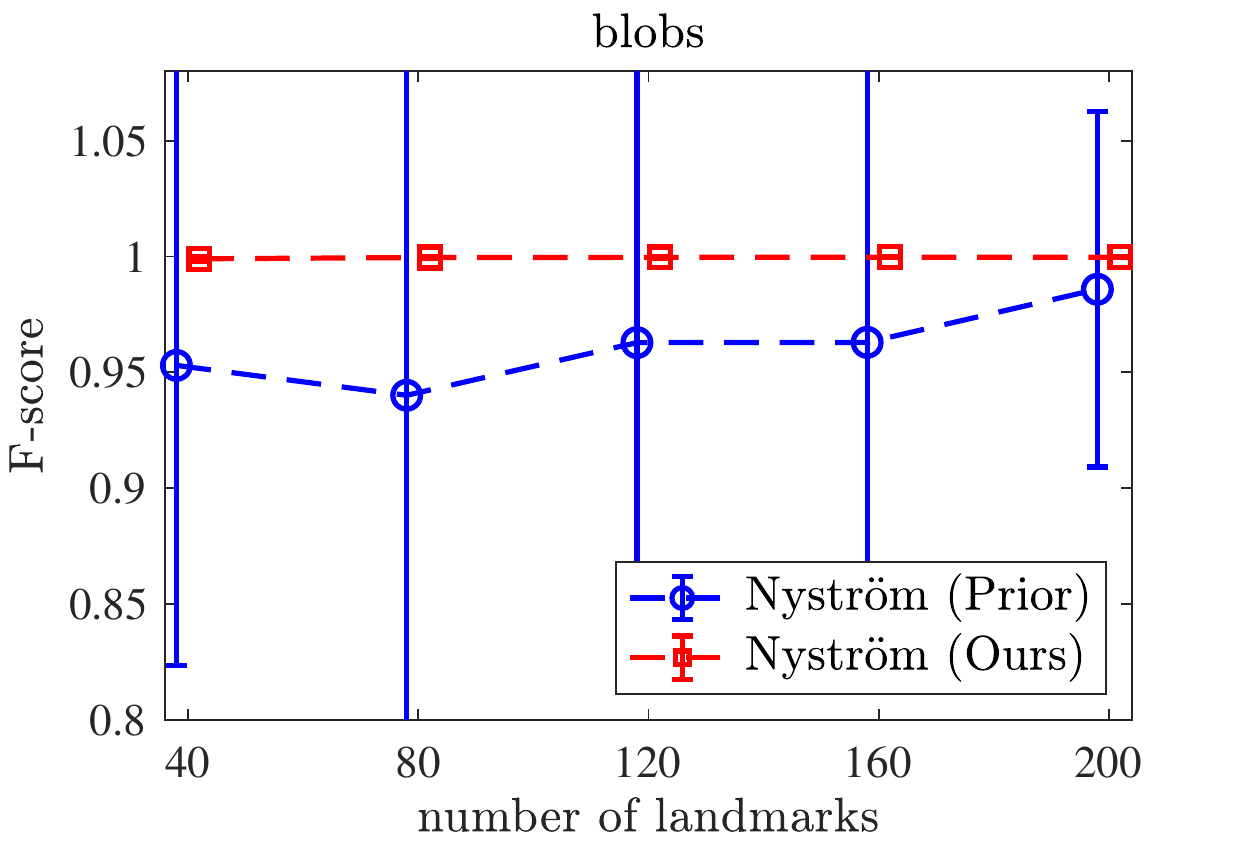}
 	}
 	
 	\subfloat[][\label{fig:nmi-moons}  moons, NMI]{
 		\includegraphics[width=0.33\linewidth]{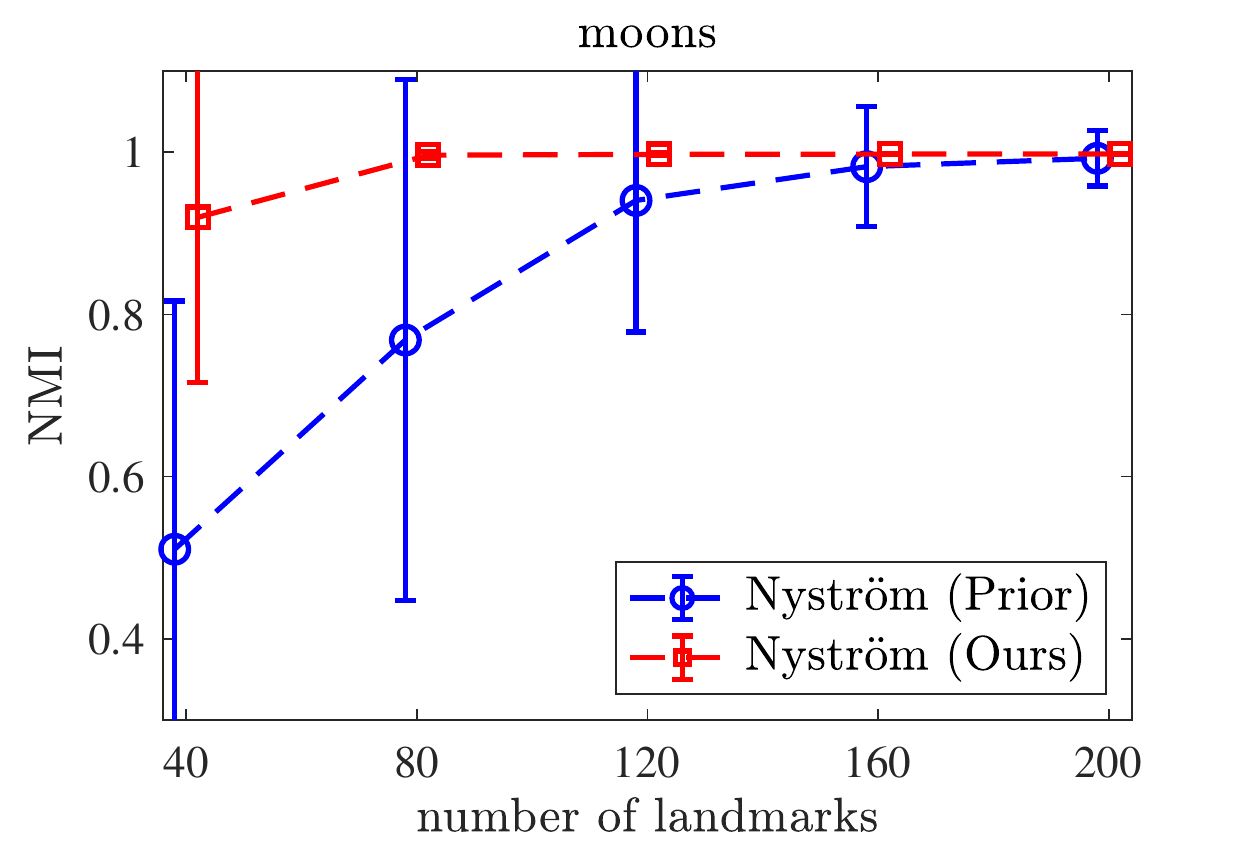}
 	}
 	\subfloat[][\label{fig:nmi-circles} circles, NMI]{
 		\includegraphics[width=0.33\linewidth]{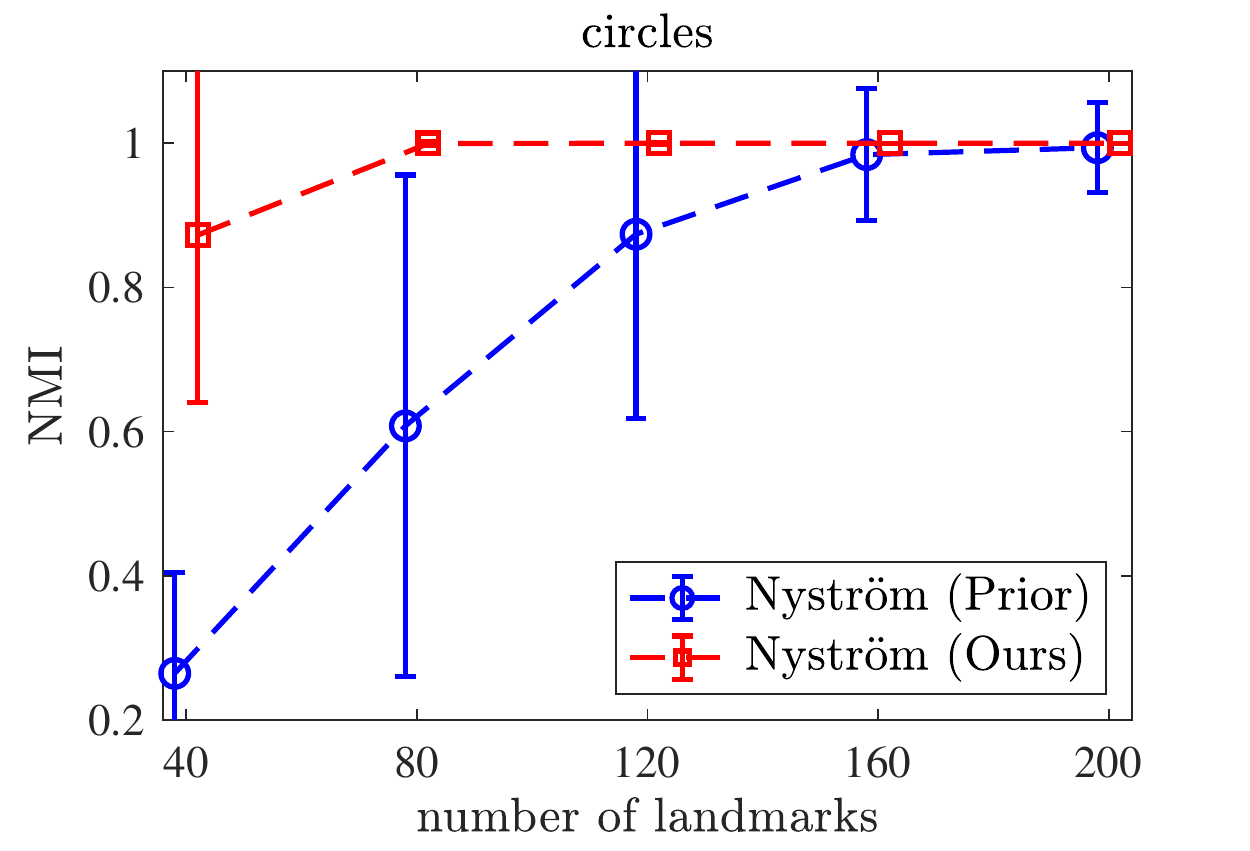}
 	}
 	\subfloat[][\label{fig:nmi-blobs} blobs, NMI]{
 		\includegraphics[width=0.33\linewidth]{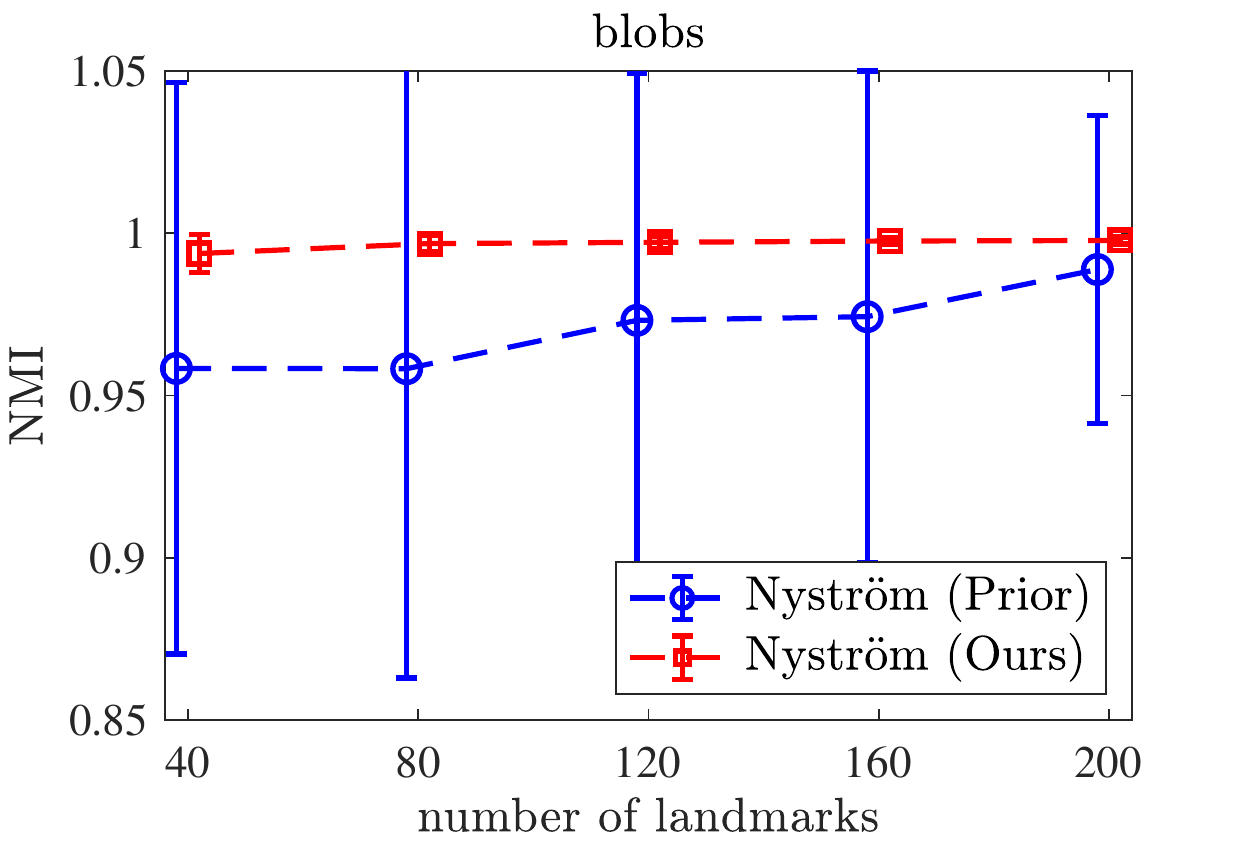}
 	}
 	
 	\caption{\label{fig:syn1}
 		Demonstrating the accuracy of our proposed Nystr\"om-based spectral clustering method on three synthetic data sets.
 	}
 \end{figure*}
 
\subsection{Experiments on Synthetic Data}
In the first experiment, we present a thorough comparison of our approach with the prior scalable Nystr\"om-based spectral clustering method that was introduced in \cite{li2011time}. As discussed in Section \ref{sec:rel}, the prior approximate spectral clustering method based on the Nystr\"om approximation \cite{fowlkes2004spectral} does not scale well with the data size due to the quadratic dependence on the number of landmarks. In our comparison, we consider three synthetic data sets covering different intrinsic structures in real data. In Fig.~\ref{fig:data-moons}, \ref{fig:data-circles}, and \ref{fig:data-blobs}, we plot the three data sets (named moons, circles, and blobs) that contain $n=100,\!000$ samples in $\mathbb{R}^2$ with up to three clusters. To be fair,  we set the kernel parameter $\sigma=0.2$ for all three data sets and both spectral clustering methods share the same $\mathbf{C}$ and $\mathbf{W}$ in each trial. 

The mean and standard deviation of the two evaluation criteria, i.e., F-score and NMI, are shown in Fig.\ref{fig:syn1} for varying values of landmarks $m$ from $40$ to $200$. For all three cases, we see that our proposed approach outperforms the prior work on Nystr\"om-based spectral clustering \cite{li2011time} with higher accuracy levels and lower standard deviations. Also, the superior performance of our approach is more significant when $m$ is relatively small, which is desirable for reducing the memory overhead associated with storing and processing the matrix $\mathbf{C}\in\mathbb{R}^{n\times m}$ obtained from the Nystr\"om approximation. For comparison, we observe that our method provides almost perfect clustering of the blobs data set with $m=40$ landmarks. On the other hand, the prior work on Nystr\"om-based spectral clustering does not have a satisfactory performance on the same data with $m=200$ landmarks, which is five times the required number of landmarks in our proposed method. 

\begin{figure*}[t]
	\centering
	\subfloat[][\label{fig:time-moons}  moons, $k=2$ clusters]{
		\includegraphics[width=0.33\linewidth]{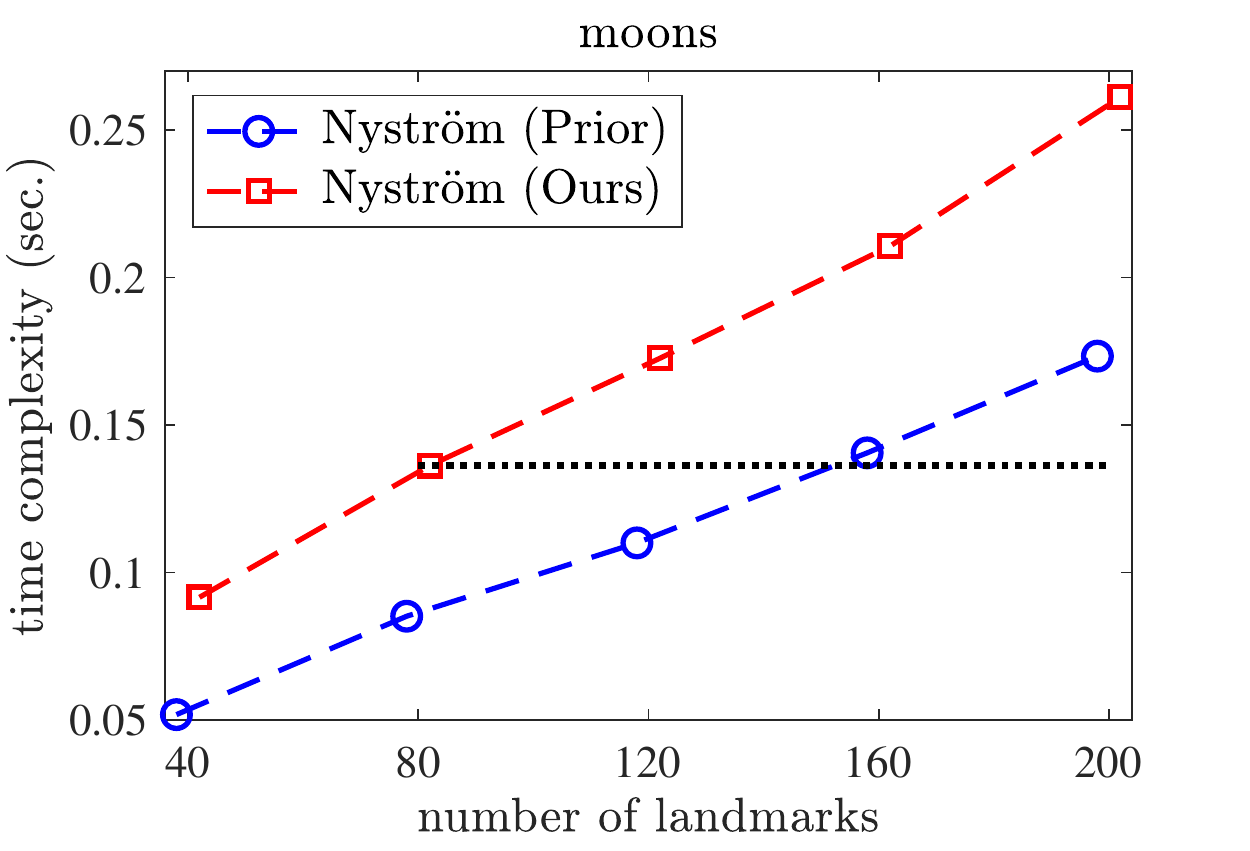}
	}
	\subfloat[][\label{fig:time-circles} circles, $k=2$ clusters]{
		\includegraphics[width=0.33\linewidth]{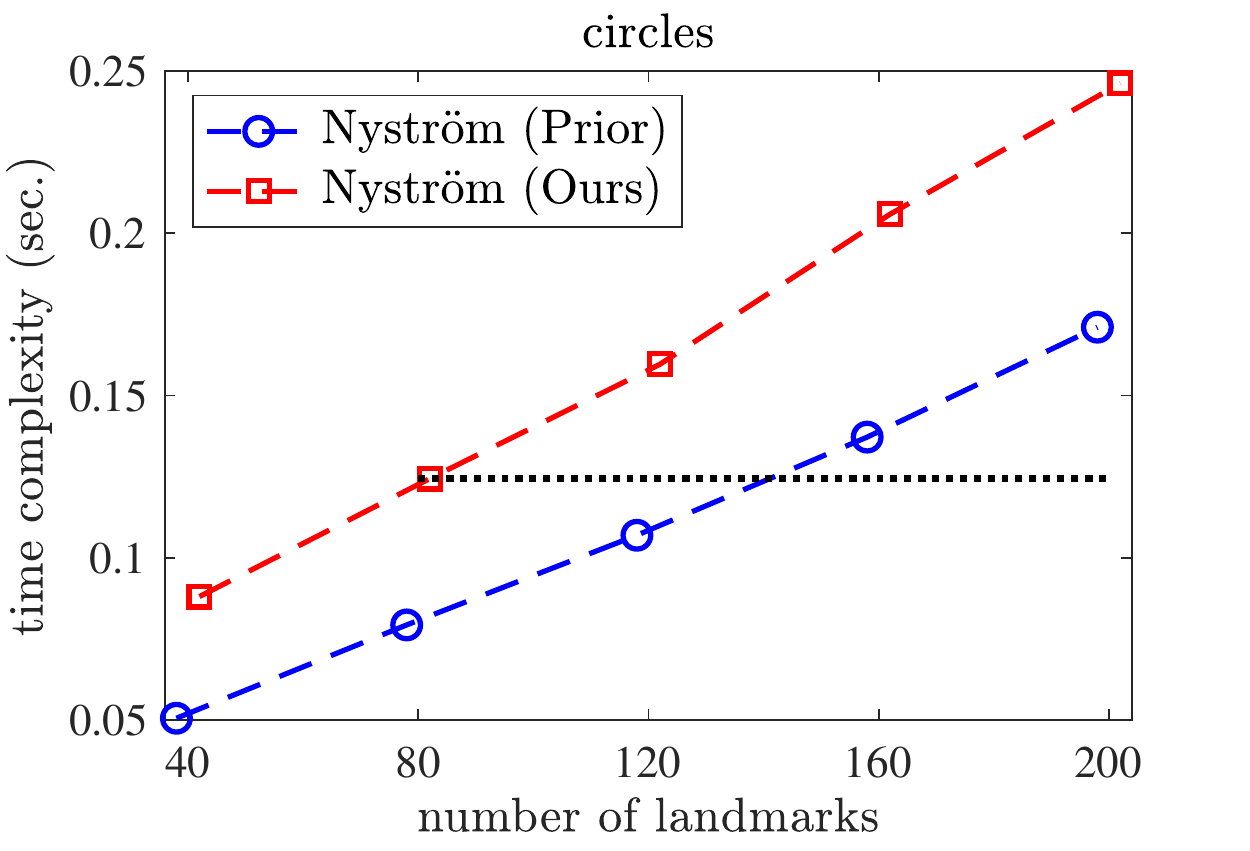}
	}
	\subfloat[][\label{fig:time-blobs} blobs, $k=3$ clusters]{
		\includegraphics[width=0.33\linewidth]{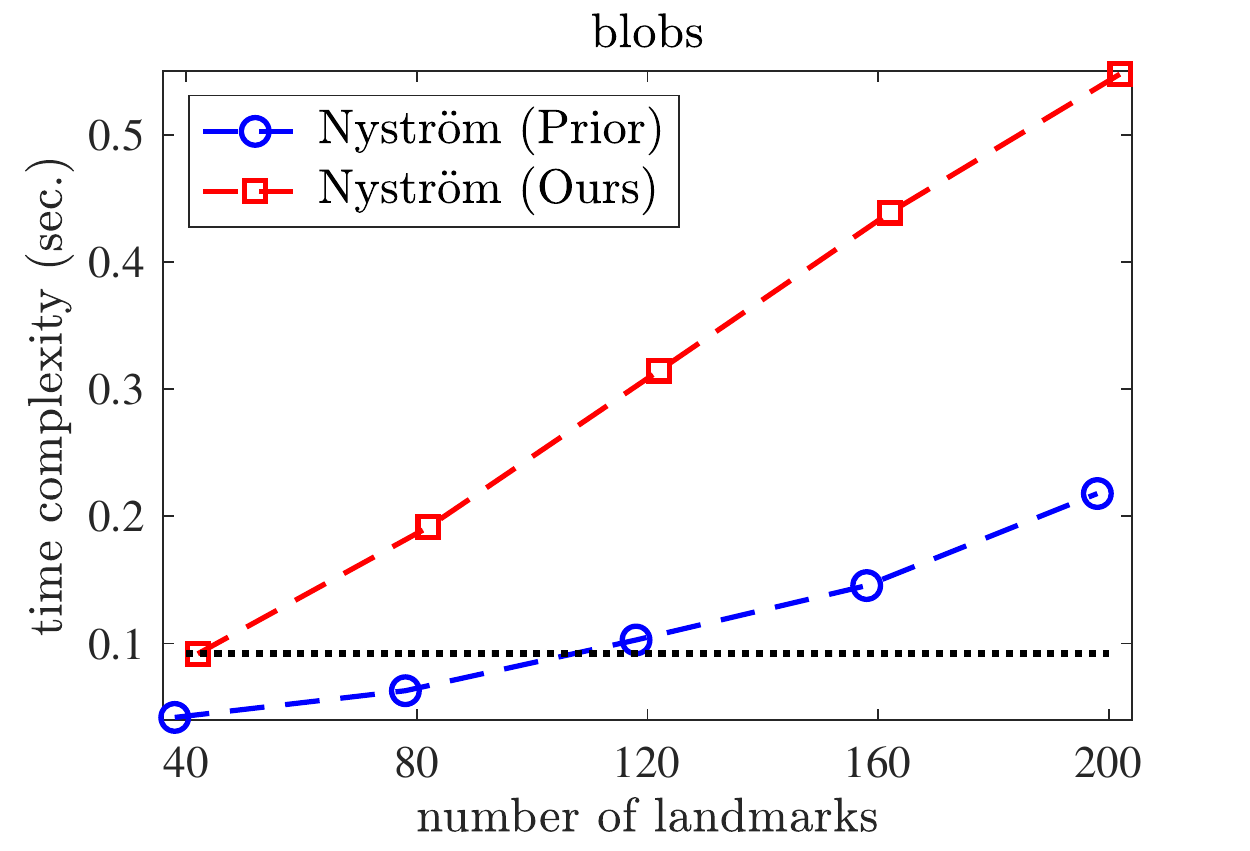}
	}
	
	\caption{\label{fig:syn1_time}
Reporting the time complexity of Nystr\"om-based spectral clustering methods on three synthetic data sets. 
	}
\end{figure*}

We also present timing results for all three data sets in Fig.~\ref{fig:syn1_time}. These results confirm that both methods' computational cost grows linearly as the number of landmarks $m$ increases. However, for a fixed value of $m$, the prior work on spectral clustering is somewhat faster than ours because of enforcing the strict rank reduction step on $\mathbf{W}$, which compromises the clustering accuracy as demonstrated in Fig.~\ref{fig:syn1}. However, our method exploits the threshold parameter $\gamma$ and the decaying spectrum of $\mathbf{W}$ to compute its rank-$l$ approximation to improve the trade-off between accuracy and efficiency. We draw a dashed horizontal line highlighting the running time of our method yielding perfect clustering results using the smallest possible landmark set. This comparison shows that our approach is more time-efficient than the prior Nystr\"om-based spectral clustering method to reach a specific accuracy. 

Although our approach's time complexity is a linear function of $m$ for all three data sets in Fig.~\ref{fig:syn1_time}, we observe that the running time of the proposed method on the blobs data set is noticeably higher than the other experiments. To explain this increase in running time, we report the average value of $l$, i.e., the rank of the inner matrix $\mathbf{W}$, in our proposed method for all three data sets in Table \ref{table:syn}. The reported results reveal that the average number of retained eigenvectors for the blobs data set is more than the first two data sets because the resulting matrix $\mathbf{W}$ has a slowly decaying spectrum. Thus, we should automatically increase the value of $l$ to capture its spectral energy, which justifies our method's higher time complexity on the blobs data set. This table also exemplifies that restricting the rank of $\mathbf{W}$ to be precise $k$, as prescribed in the prior Nystr\"om-based spectral clustering, leads to significant information loss when the spectrum of $\mathbf{W}$ decays slowly. 

\begin{table}[h]
	\centering
	\caption{Average of the rank parameter $l$ for fixed $\gamma=10^{-2}$ on three synthetic data sets.}
	\label{table:syn}
	\setlength{\tabcolsep}{5pt}
	\begin{tabular}{cccccc}
		\toprule
		data set& $m=40$ & $m=80$ & $m=120$ & $m=160$ & $m=200$  \\ 		\midrule
		moons   & $35.4$ & $55.6$ & $66.0$  & $72.3$  & $77.5$   \\
		\addlinespace
		circles & $34.1$ & $52.7$ & $63.2$  & $69.7$  & $74.1$  \\
		\addlinespace
		blobs   & $37.9$ & $66.5$ & $86.9$  & $101.2$ & $112.1$\\ \bottomrule
	\end{tabular}
\end{table}

We further investigate the influence of the threshold parameter $\gamma$ on the performance and time complexity of our approach using the blobs data set, which is more challenging than the other synthetic data sets we considered in this section. We fix the number of landmarks $m=200$ and consider four different values of $\gamma\in\{10^{-3},5\times10^{-3},10^{-2},5\times10^{-2}\}$. We report clustering accuracy and running time results in Table \ref{table:syn2}. The proposed spectral clustering method correctly identifies the $k=3$ clusters within the blobs data set for all values of $\gamma\leq 10^{-2}$. Based on this result and similar observations, we noticed that parameter tuning is not crucial, and setting $\gamma=10^{-2}$ provides a reasonable trade-off between accuracy and efficiency throughout our experiments. However, fine-tuning the parameter $\gamma$ based on the spectrum of $\mathbf{W}$ will provide a significant reduction of time complexity on massive data sets because the cost is inversely proportional to $\gamma$. 

\begin{table}[h]
	\centering
	\caption{Accuracy and running time for varying values of $\gamma$ and fixed $m=200$ on the blobs data set.}
	\label{table:syn2}
	\setlength{\tabcolsep}{5pt}
	\begin{tabular}{ccccc}
		\toprule
		& $\gamma=10^{-3}$ & $\gamma=5\times10^{-3}$ & $\gamma=10^{-2}$ & $\gamma=5\times 10^{-2}$   \\ 		\midrule
		F-score   & $1.00$ & $1.00$ & $1.00$  & $0.96\pm0.13$    \\
		\addlinespace
		NMI & $1.00$ & $1.00$ & $1.00$  & $0.97\pm0.09$   \\
		\addlinespace
		time (sec.)   & $0.74$ & $0.61$ & $0.54$  & $0.39$ \\ \bottomrule
	\end{tabular}
\end{table}

Next, we compare our proposed spectral clustering method based on the Nystr\"om approximation with another competing technique that circumvents the construction of full similarity graphs. We consider a method called Random Features \cite{wu2018scalable}, aiming to directly approximate the Gaussian kernel function based on sampling from an explicit feature map. This method constructs an approximation of the modified kernel matrix $\mathbf{M}\approx\mathbf{S}\mathbf{S}^T$ with $\mathbf{S}\in\mathbb{R}^{n\times D}$, resembling that of our approach. However, the introduced technique in \cite{wu2018scalable}  should choose $D$ so that it significantly exceeds the ambient dimension of the original data. To alleviate the cost, this method enforces each row of $\mathbf{S}$ to be sparse with a few non-zero entries, and a state-of-the-art sparse SVD solver is employed. The sparsity of each row of $\mathbf{S}$ is selected to be $m$  for a fair comparison with spectral clustering methods that utilize the Nystr\"om approximation with $m$ landmarks to approximate the kernel matrix.  We use the provided code by the authors, which is implemented in Matlab and uses C Mex functions for computationally expensive parts. 
\begin{figure*}[t]
	\centering
	\subfloat[][\label{fig:rb_fsc}  F-score]{
		\includegraphics[width=0.33\linewidth]{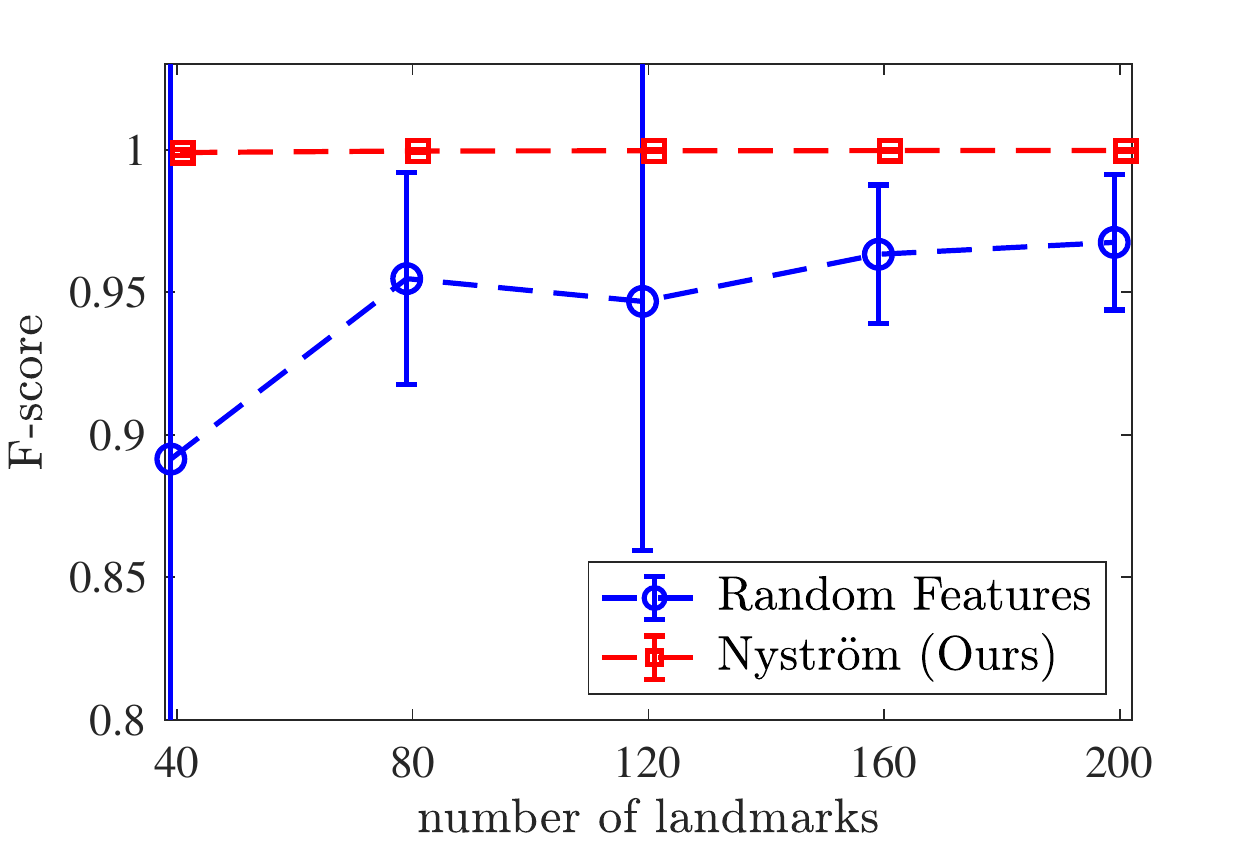}
	}
	\subfloat[][\label{fig:rb_nmi} NMI]{
		\includegraphics[width=0.33\linewidth]{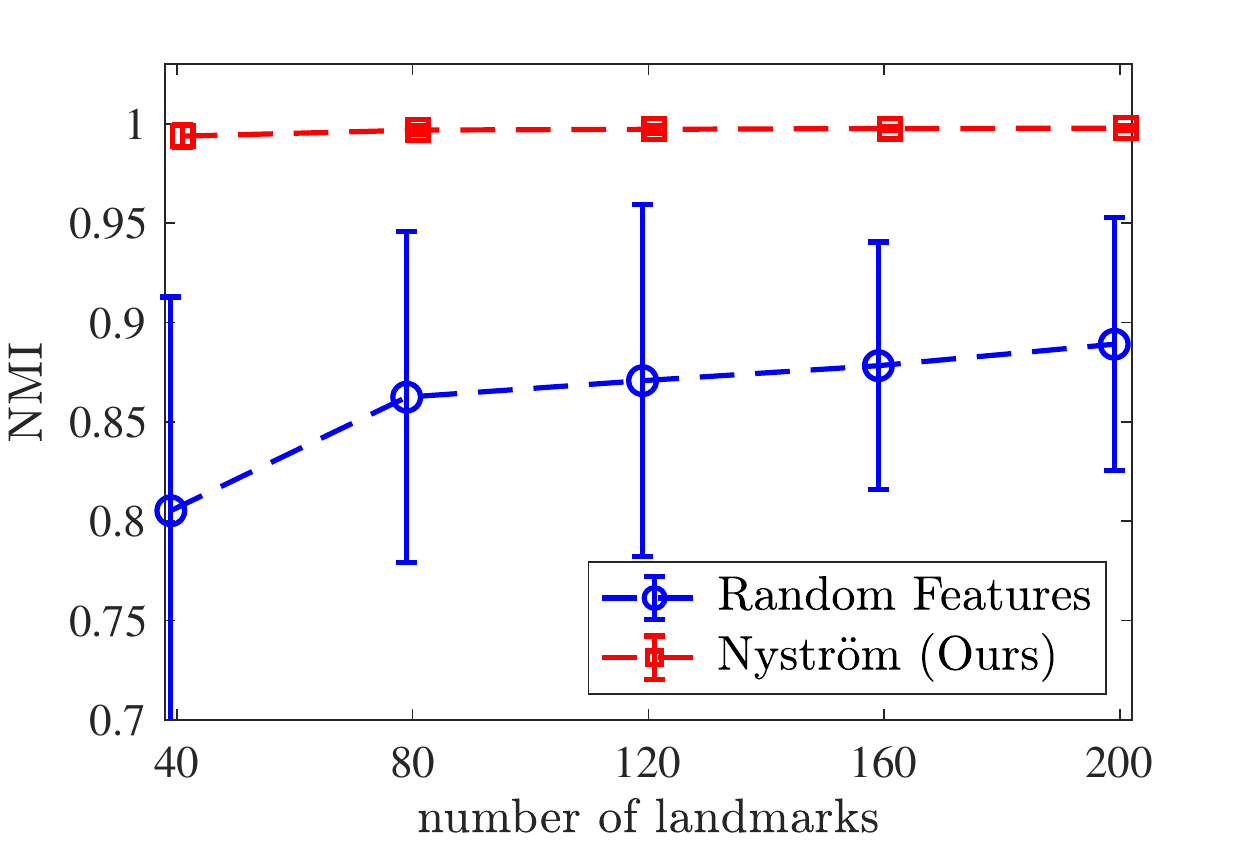}
	}
	\subfloat[][\label{fig:rb_time} time complexity]{
		\includegraphics[width=0.33\linewidth]{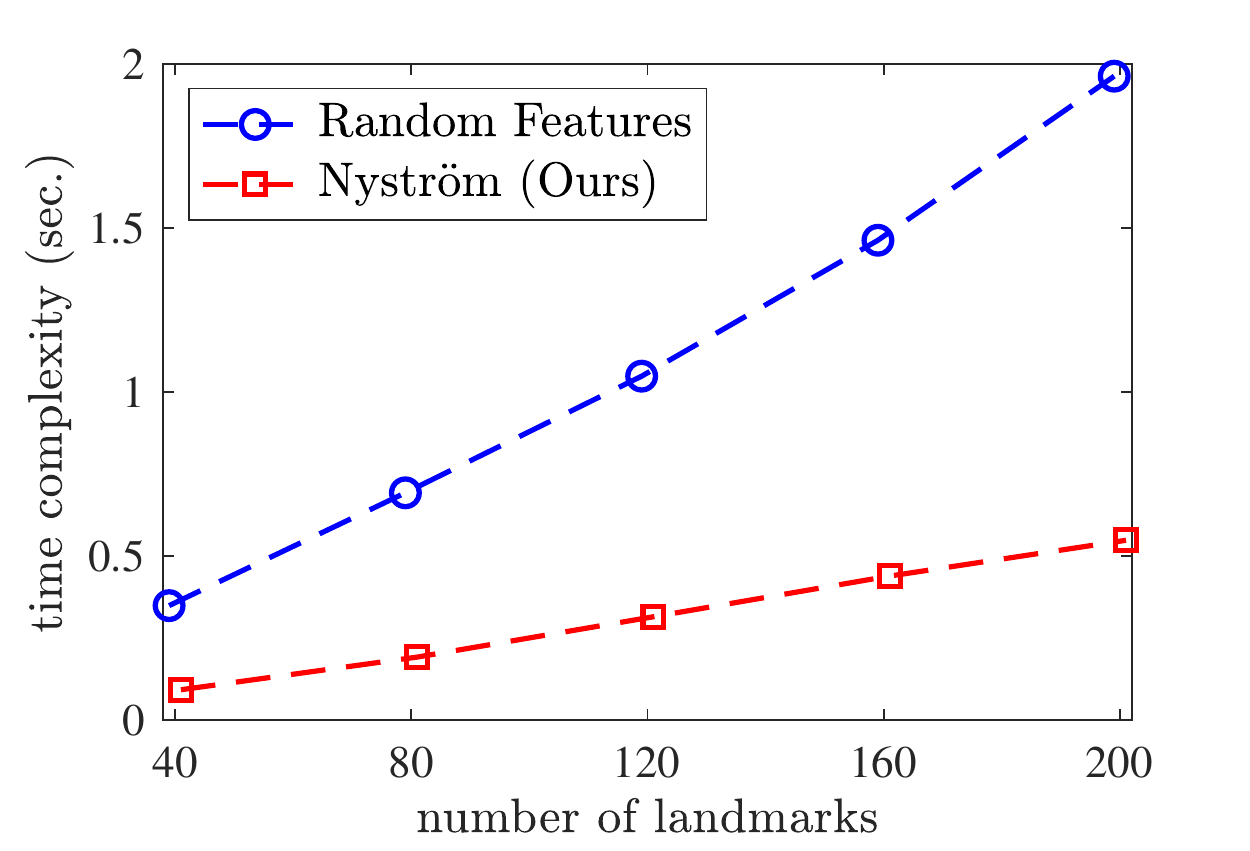}
	}
	
	\caption{\label{fig:syn2_rb}
		Using the blobs data set to compare our proposed method with the previous work on spectral clustering based on Random Features. 
	}
\end{figure*}

In Fig.~\ref{fig:syn2_rb}, we report  accuracy and time complexity results for various values of $m$ on the blobs data set using the same kernel parameter $\sigma=0.2$. We see that our approach noticeably outperforms Random Features for all values of $m$. Furthermore, our approach is more efficient, and the computational savings become more prominent as we increase $m$. We also compared the Random Features technique with our introduced method on the other synthetic data sets omitted to save space, revealing the poor performance of Random Features. The unsatisfactory clustering accuracy of Random Features in our comparison may be related to the lack of careful adjustment of parameters such as the kernel parameter $\sigma$ and the number of landmarks $m$. For example, the reported results in the original work \cite{wu2018scalable} primarily focused on a relatively large value of $m$, such as $m=1,\!024$. Another recent work also observed that using small values of $m$ adversely affects the performance of approximate spectral clustering techniques based on Random Features \cite{he2018fast}. However, our method successfully clusters the blobs data set in less time with $m=40$ landmarks without tuning $\sigma$ for each experiment separately. A further advantage of our approach is that it does not require SVD solvers specifically tailored to large sparse matrices.

\subsection{Experiments on Real Data}
We investigate the performance and computational cost of our proposed spectral clustering method on two public benchmark data sets. For our comparison, we use the prototypical spectral clustering algorithm depicted in Alg.~\ref{alg:SC} to obtain baselines for both evaluation metrics, i.e., F-score and NMI. Therefore, we consider smaller data sets to run the full spectral clustering algorithm on a single machine. Like the previous section, we compare against the prior scalable spectral clustering technique based on the Nystr\"om method \cite{li2011time}. We use the Matlab built-in spectral clustering algorithm, called ``spectralcluster,'' with default values to perform spectral clustering on the full data set. The Matlab implementation also allows performing approximate spectral clustering by forming a sparse nearest neighbor similarity graph \cite{von2007tutorial}. For this technique, we explore the influence of the number of nearest neighbors on the clustering accuracy and running time.

As mentioned, we consider two real data sets from LIBSVM \cite{CC01a}. The MNIST data set is a  collection of handwritten digit images, where a $d=28^2=784$ dimensional vector represents each image. In this work, we reduce the dimension of the original data to $d=500$ by using principal component analysis (PCA) \cite{anaraki2014memory}. We also consider two subsets of this data set: (1) $k=2$ classes with $n=11,\!800$ samples and (2) $k=3$ classes with $n=17,\!718$ samples. The mushrooms data set contains $n=8,\!124$ samples in $\mathbb{R}^{112}$. Each sample in the data set contains information that describes the physical characteristics of a single mushroom. The ground-truth labels provide information regarding poisonous or edible mushrooms; we thus have $k=2$ clusters. Since these two data sets have a different range of values and structures, we used the Matlab built-in support vector machine classifier to find an appropriate value for the kernel parameter $\sigma$. While this parameter tuning requires the ground-truth labels, we do not use them when performing spectral clustering. We set the threshold parameter $\gamma=10^{-2}$, which is the same value used in the previous section for all synthetic data sets.

\begin{table*}[ht]
	\centering
		\caption{Comparing accuracy and running time on real data sets for $m=40$ and $m=80$.}
	\label{table:real}
		\setlength{\tabcolsep}{10pt}
	\begin{tabular}{@{}lccccc@{}}
		\toprule
		& Full Spectral Clustering & Ours, $m=40$  &  Previous, $m=40$ & Ours, $m=80$  & Previous, $m=80$ \\ 
		\midrule
		\textbf{MNIST (2 classes)}\\
		F-score  & $0.976$    & $0.974\pm0.002$    & $0.949\pm0.047$ & $0.976\pm0.001$& $0.961\pm 0.022$   \\
		NMI    & $0.838$     & $0.829\pm0.010$     & $0.734\pm0.145$     & $0.834\pm0.005$& $0.770\pm0.089$\\
		time (sec.)   & $77.194$     & $0.070$     & $0.064$     & $0.105$& $0.094$\\
		\midrule
		\textbf{MNIST (3 classes)}\\
		F-score  & $0.976$     & $0.941\pm0.029$    &  $0.540\pm0.241$ & $0.957\pm 0.021$ & $0.637\pm0.276$\\
		NMI    &  $0.895$   & $0.797\pm 0.067$     & $0.492\pm 0.272$    & $0.838\pm 0.052$ & $0.558\pm 0.297$\\
		time (sec.)   & $246.511$     & $0.112$    & $0.102$    & $0.160$ & $0.140$\\
		\midrule
		\textbf{mushrooms}\\
		F-score   & $0.891$    & $0.888\pm0.004$     & $0.804\pm 0.099$   & $0.890\pm0.001$ & $0.828\pm0.096$  \\
		NMI     &   $0.566$  & $0.551\pm 0.019$     & $0.428\pm 0.135$   & $0.562\pm 0.005$ & $0.462\pm0.125$  \\
		time (sec.)    &  $27.216$    & $0.019$     & $0.016$   & $0.025$ & $0.018$ \\
		\bottomrule                          
	\end{tabular}
\end{table*}

Table \ref{table:real} reports the mean and standard deviation of clustering accuracy as well as time complexity for two values of $m=40$ and $m=80$. As we see, our proposed approach consistently outperforms the prior work on Nystr\"om-based spectral clustering. Our method often reaches the accuracy of spectral clustering on the full data set with $m=40$ landmarks while reducing the running time by two to three orders of magnitude. As expected, for a fixed value of $m$, our method is slightly slower than the previous Nystr\"om-based spectral clustering technique. However, our method is more time-efficient when our goal is to reach higher accuracy levels close to those of the full spectral clustering algorithm. Also, we report the average value of the rank parameter $l$ in Table \ref{table:reall}. These results confirm that the rank-$k$ approximation of the similarity matrix $\mathbf{W}$ leads to the inferior performance of the previous Nystr\"om-based spectral clustering method as the spectrum of $\mathbf{W}$ decays slowly. Because of overlapping clusters, we see that the expected value of $l$ significantly exceeds the number of clusters $k$. Therefore, the proposed method provides an improved trade-off between accuracy and efficiency in all experiments with real data. 

\begin{table}
	\centering
	\caption{Average of the rank parameter $l$ on three real data sets.}
	\label{table:reall}
	\setlength{\tabcolsep}{10pt}
	\begin{tabular}{cccc}
		\toprule
		data set& $m=40$ & $m=80$  \\ 		\midrule
		MNIST (2 classes)   & $12.66$ & $8.86$    \\
		\addlinespace
		MNIST (3 classes) & $39.10$ & $57.62$   \\
		\addlinespace
		mushrooms   & $20.46$ & $16.22$ \\ \bottomrule
	\end{tabular}
\end{table}

Moreover, we used the mushrooms data set to evaluate the number of nearest neighbors for forming sparse similarity graphs to reduce the complexity of prototypical spectral clustering. The Matlab's default value for the number of neighbors is $\lfloor\log(n) \rfloor$. Thus, we considered three values $\alpha\lfloor\log(n) \rfloor$, where $\alpha\in\{1,10,100\}$. The mean F-score values are $0.75$, $0.37$, and $0.63$, respectively. This result shows that increasing the number of nearest neighbors does not necessarily guarantee an improved clustering quality. Also, the F-score for the default value is substantially lower than the one obtained via performing the full spectral clustering algorithm. In terms of running time, the spectral clustering method based on sparse similarity graphs takes $2.53$, $14.5$, and $26.3$ seconds, respectively. Thus, in our comparison, another benefit of using the Nystr\"om approximation is the substantial reduction of time complexity compared to sparse similarity graphs.

\section{Conclusion} \label{sec:conc}
In this paper, we presented a principled approximate spectral clustering algorithm that builds on the Nystr\"om approximation and provided accuracy-efficiency trade-offs. The proposed method has shown to outperform existing Nystr\"om-based spectral clustering methods through several detailed experiments on synthetic and real data. This work also presented two new theoretical results to understand the influence of leveraging low-rank kernel matrices in the context of spectral clustering and normalized cut. We envision several future research directions to improve our proposed approach's performance, efficiency, and robustness. While uniform sampling provides satisfactory accuracy and scalability, we suspect that recent non-uniform sampling strategies, such as coresets, can offer improved accuracy-efficiency trade-offs.
Additionally, geometry-preserving sampling can be instrumental when the number of clusters is enormous. Also, utilizing non-uniform sampling mechanisms with strong theoretical guarantees can be integrated with our presented perturbation analysis to achieve tighter upper bounds. Another future research direction is to extend the current work to handle noisy data sets in spectral clustering.

\bibliographystyle{IEEEtran}
\bibliography{sample.bib}

\begin{IEEEbiography}[{\includegraphics[width=1in,height=1.2in,clip,keepaspectratio]{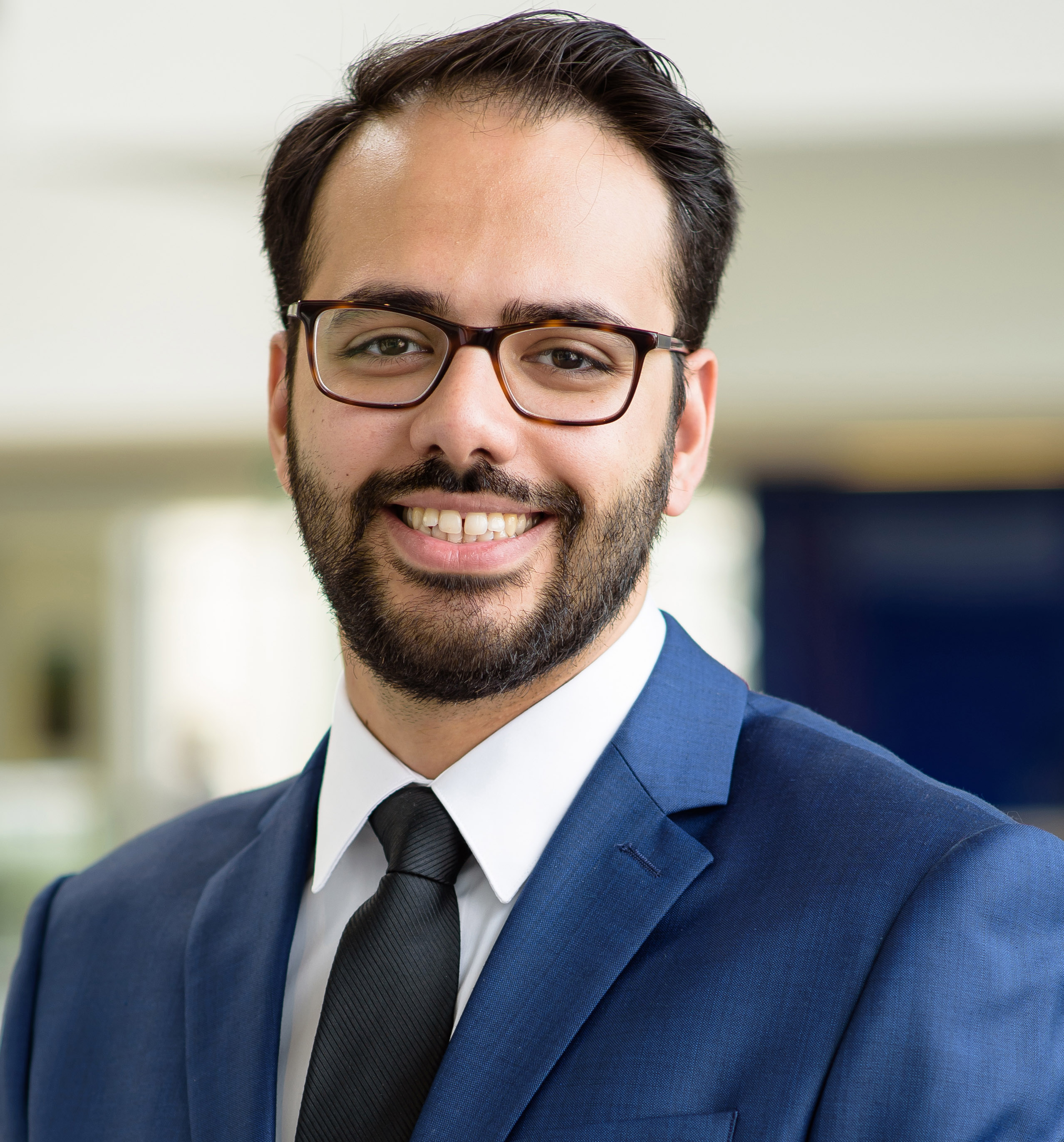}}]{Farhad Pourkamali-Anaraki} is currently an Assistant Professor in the Department of Computer Science at the University of Massachusetts Lowell. He received his PhD in Electrical Engineering from the University of Colorado Boulder in 2017. He spent one year as a Postdoctoral Research Associate in the Applied Mathematics Department at the same institution. His research focuses on developing scalable and provable machine learning algorithms for analyzing complex, high-dimensional data sets. His research also centers around the development and use of artificial intelligence and machine learning in the context of computational decision support for complex systems, including reliability analysis of critical infrastructure sectors and advanced manufacturing technologies.
\end{IEEEbiography}

\end{document}